\documentclass{article}

\usepackage{amsmath,amsfonts,bm}

\def\eqref#1{equation~\ref{#1}}

\def\1{\bm{1}}

\DeclareMathAlphabet{\mathsfit}{\encodingdefault}{\sfdefault}{m}{sl}
\SetMathAlphabet{\mathsfit}{bold}{\encodingdefault}{\sfdefault}{bx}{n}

\def\sA{{\mathbb{A}}}

\DeclareMathOperator*{\argmin}{arg\,min}

\usepackage[hidelinks, colorlinks, linkcolor=black, citecolor=blue, urlcolor=black]{hyperref}
\usepackage{url}
\usepackage{color}
\usepackage{soul}
\usepackage{amsmath,amssymb,amsthm, mathtools}
\usepackage{subfigure}
\usepackage[final]{graphicx}
\usepackage{xcolor}
\usepackage{natbib}
\usepackage{mathrsfs}
\usepackage{algorithm, algorithmic}

\renewcommand{\cal}{\mathcal}
\newcommand\cA{{\mathcal A}}
\newcommand\cB{{\mathcal B}}

\newcommand{\cX}{{\cal X}}
\newcommand{\cY}{{\cal Y}}
\newcommand{\cD}{{\cal D}}

\newcommand\cH{{\mathcal H}}

\newcommand{\cN}{{\cal N}}

\newcommand{\cP}{{\cal P}}

\newcommand{\cF}{{\mathcal{ F}}}
\newcommand{\cZ}{{\mathcal{ Z}}}

\newcommand{\bE}{\mathbb{E}}

\newcommand{\bP}{\mathbb{P}}

\newcommand{\bR}{{\mathbb R}}

\renewcommand{\leq}{\leqslant}
\renewcommand{\geq}{\geqslant}

\newtheorem{theorem}{Theorem}[section]
\newtheorem{proposition}{Proposition}[section]
\newtheorem{lemma}{Lemma}[section]
\newtheorem{corollary}{Corollary}[section]
\newtheorem{remark}{Remark}[section]
\newtheorem{assumption}{Assumption}[section]

\newtheorem{definition}{Definition}[section]
\newtheorem{example}{Example}[section]

\def\sA{\mathscr{A}}
\def\htheta{\hat{\theta}}
\renewcommand{\backslash}{-}
\renewcommand{\textnormal}[1]{#1}

\usepackage{hyperref}

 \usepackage[accepted]{icml2021}

\newcommand{\hangfeng}[1]{\textcolor{brown}{[HF: #1]}}
\newcommand\hfch[1]{\textrm{\color{brown} #1}}

\usepackage[accepted]{icml2021}

\icmltitlerunning{Generalization Bounds with Locally Elastic Stability}

\begin{document}

\twocolumn[
\icmltitle{Toward Better Generalization Bounds with Locally Elastic Stability}




\begin{icmlauthorlist}
\icmlauthor{Zhun Deng}{a}
\icmlauthor{Hangfeng He}{b}
\icmlauthor{Weijie J.~Su}{c}
\end{icmlauthorlist}

\icmlaffiliation{a}{Harvard University}
\icmlaffiliation{b}{Department of Computer and Information Science, University of Pennsylvania}
\icmlaffiliation{c}{Wharton Statistics Department, University of Pennsylvania}

\icmlcorrespondingauthor{Zhun Deng}{zhundeng@g.harvard.edu}

\icmlkeywords{Machine Learning, ICML}

\vskip 0.3in
]



\printAffiliationsAndNotice{}  

\begin{abstract}
Algorithmic stability is a key characteristic to ensure the generalization ability of a learning algorithm. Among different notions of stability, \emph{uniform stability} is arguably the most popular one, which yields exponential generalization bounds. However, uniform stability only considers the worst-case loss change (or so-called sensitivity) by removing a single data point, which is distribution-independent and therefore undesirable. There are many cases that the worst-case sensitivity of the loss is much larger than the average sensitivity taken over the single data point that is removed, especially in some advanced models such as random feature models or neural networks. Many previous works try to mitigate the distribution independent issue by proposing weaker notions of stability, however, they either only yield polynomial bounds or the bounds derived do not vanish as sample size goes to infinity. Given that, we propose \emph{locally elastic stability} as a weaker and distribution-dependent stability notion, which still yields exponential generalization bounds. We further demonstrate that locally elastic stability implies tighter generalization bounds than those derived based on uniform stability in many situations by revisiting the examples of bounded support vector machines, regularized least square regressions, and stochastic gradient descent.

 \end{abstract}

\section{Introduction}
A central question in machine learning is how the performance of an algorithm on the training set carries over to unseen data. Continued efforts to address this question have given rise to numerous generalization error bounds on the gap between the population risk and empirical risk, using a variety of approaches from statistical learning theory \citep{vapnik1979estimation,vapnik2013nature,bartlett2002rademacher,bousquet2002stability}. Among these developments, algorithmic stability stands out as a general approach that allows one to relate certain specific properties of an algorithm to its generalization ability. Ever since the work of \citet{devroye1979distribution}, where distribution-independent exponential generalization bounds for the concentration of the \text{leave-one-out} estimate are proposed, various results for different estimates are studied. \citet{lugosi1994posterior} study the smooth estimates of the error for the deleted estimate developed in terms of a posterior distribution and \citet{kearns1999algorithmic} propose \emph{error stability}, which provides sanity-check bounds for more general classes of learning rules regarding the deleted estimate. For general learning rules, \citet{bousquet2002stability} propose the notion of \emph{uniform stability}, which extends \citet{lugosi1994posterior}'s work and yields exponential generalization bounds. Loosely speaking, \citet{bousquet2002stability} show that an algorithm would generalize well to new data if this algorithm is uniformly stable in the sense that its loss function is not sensitive to the deletion of a single data point. To date, uniform stability is perhaps the most popular stability notion.

Despite many recent developments, most results on stability and generalization can be divided into two categories if not counting sanity-check bounds. The first category includes stability notions such as hypothesis stability, which only yield sub-optimal \textit{polynomial} generalization bounds. The second category includes stability notions based on uniform stability and its variants, which yield optimal \textit{exponential} generalization bounds. Nevertheless, the stability notions in the second category either stop short of providing distribution-dependent bounds or, worse, the bounds do \textit{not} vanish even when the training sample size tends to infinity~\cite{abou2019exponential}. Recognizing these facts, in this paper, we aim to relax the uniform stability notion and propose a weaker and distribution-dependent stability notion, which yields exponential generalization bounds that are consistent in the sense that the bounds vanish to zero as the training sample size tends to infinity.
\begin{figure*}[t]
\centering
		 \subfigure[Sensitivity of neural networks.]{			\centering			\includegraphics[scale=0.33]{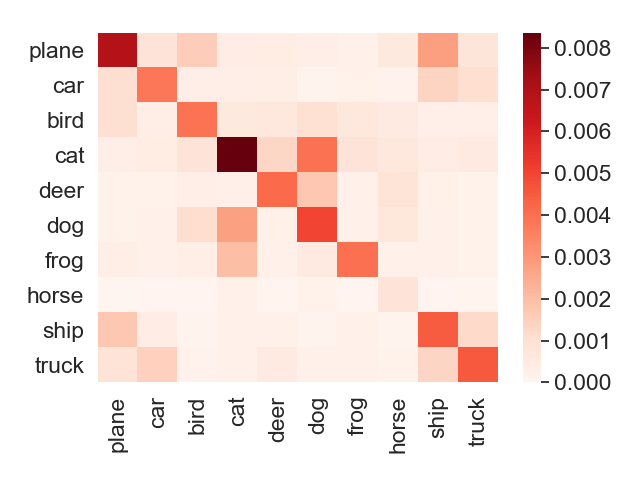}
		\label{fig:NNs-IF}}
		\subfigure[Sensitivity of a random feature model.]{
			\centering
			\includegraphics[scale=0.33]{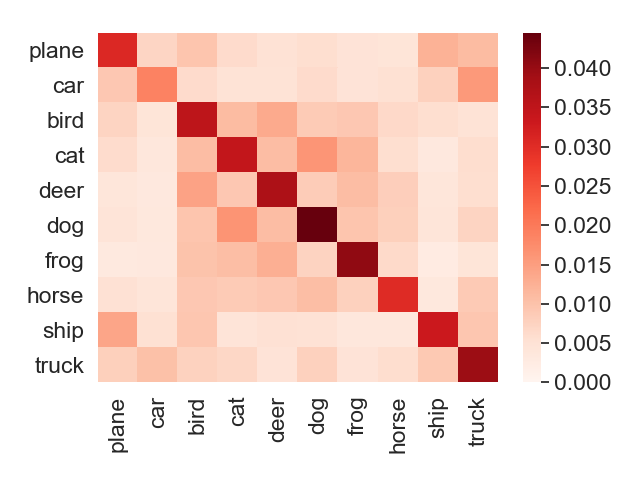}
			\label{fig:RF-IF}}
      \subfigure[Sensitivity of a linear model.]{
			\centering
		\includegraphics[scale=0.33]{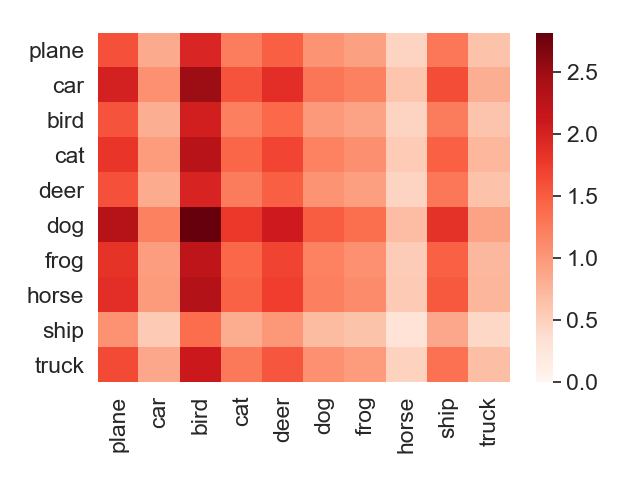}
		\label{fig:Linear-IF}}
		\caption{Class-level sensitivity approximated by influence functions for 
		neural networks (based on a pre-trained $18$-layer ResNet), a random feature model (based on a randomly initialized $18$-layer ResNet),
		and a linear model on CIFAR-10. The vertical axis denotes the classes in the test data and the horizontal  axis denotes the classes in the training data. The class-level sensitivity from class $a$ in the training data to class $b$ in the test data is defined as $C(c_a, c_b) = \frac{1}{|S_a| \times |\tilde{S}_b|}\sum_{z_i \in S_a} \sum_{z \in \tilde{S}_b}|l(\hat{\theta},z)-l(\hat{\theta}^{\backslash i},z)|$, where $S_a$ denotes the set of examples from class $a$ in the training data and $\tilde{S}_b$ denotes set of examples from class $b$ in the test data. 
		}
		\label{fig:IF-results}
\end{figure*}

\subsection{A Motivating Example}
To further motivate our study, note that there are many cases where the worst-case sensitivity of the loss is much larger than the average sensitivity, especially in random feature models or neural networks. As a concrete example, from Figure~\ref{fig:IF-results}, we can observe that the sensitivity of neural networks and random feature models depends highly on the label information. To be precise, consider training two models on the CIFAR-10 dataset \citep{krizhevsky2009learning} and another dataset obtained by removing one training example, say an image of a plane, from CIFAR-10, respectively. Figure~\ref{fig:IF-results} shows that the difference between the loss function values for the two models \textit{depends} on the label of the test image that the loss function is evaluated at: the difference between the loss function values, or sensitivity for short, is significant if the test image is another plane, and the sensitivity is small if the test image is from a different class, such as car or cat. Concretely, the average plane-to-plane difference is about seven times the average plane-to-cat difference. The dependence on whether the two images belong to the same class results in a pronounced diagonal structure in Figure~\ref{fig:NNs-IF}, which is consistent with the phenomenon of \textit{local elasticity} in deep learning training \citep{he2020local,chen2020label}. In particular, this structural property of the loss function differences clearly demonstrates that uniform stability fails to capture how sensitive the loss function is in the \textit{population} sense, which is considerably smaller than the worst-case sensitivity, for the neural networks and random feature models.
\subsection{Our Contribution}
As our first contribution, we introduce a new notion of algorithmic stability that is referred to as \textit{locally elastic stability} to take into account the message conveyed by Figure~\ref{fig:IF-results}. This new stability notion imposes a data-dependent bound on the sensitivity of the loss function, as opposed to a constant bound that uniform stability and many of its relaxations use.

The second contribution of this paper is to develop a generalization bound for any locally elastically stable algorithm. This new generalization bound is obtained by a fine-grained analysis of the empirical risk, where using McDiarmid's inequality as in \citet{bousquet2002stability} no longer works. Specifically, we expect the empirical sum of the sensitivities by deleting different samples to be close to the expected sensitivity taken over the deleted sample. However, conditioning on that event, the dependency among input examples invalidate McDiarmid's inequality. To overcome this difficulty, we develop novel techniques that allow us to obtain a sharper analysis of some important quantities. Our results show that the generalization error is, loosely speaking, upper bounded by the expectation of the sensitivity function associated with locally elastic stability over the population of training examples. Assuming uniform stability, however, classical generalization bounds are mainly determined by the largest possible sensitivity over all pairs of training examples.  We further demonstrate that our bounds are tighter than those derived based on uniform stability in many situations by revisiting the examples of bounded support vector machines (SVM), regularized least square regressions, and stochastic gradient descent (SGD). Although it requires further exploration on how to make the new bounds applicable to deep learning models in practice, the insights from this new stability notion shall shed light on the development of future approaches toward demystifying the generalization ability of modern neural networks.

\subsection{Related Work}
Ever since \citet{kearns1999algorithmic} and \citet{bousquet2002stability} proposed the notions of uniform stability and hypothesis stability, a copious line of works has been devoted to extending and elaborating on their frameworks. In \citet{mukherjee2006learning}, \citet{shalev2010learnability} and \citet{kutin2002almost}, the authors show there exist cases where stability is the key necessary and sufficient condition for learnability but uniform convergence is not.
On one hand, error stability is not strong enough to guarantee generalization \citep{kutin2012almost}. On the other hand, hypothesis stability guarantees generalization but only provides polynomial tail bounds. Fortunately, uniform stability guarantees generalization and further provides exponential tail bounds. In \citet{feldman2018generalization}, the authors develop the generalization bound for the cases where uniform stability parameter is of order $\Omega(1/\sqrt{m})$, 
where $m$ is the sample size. In subsequent work, \citet{feldman2019high} prove a nearly tight high probability bound for any uniformly stable algorithm. In \citet{bousquet2020sharper}, the authors provide sharper bounds than \citet{feldman2019high} and also provide general lower bounds which can be applied to certain generalized concentration inequalities. There are also works seeking to relax uniform stability such as \cite{abou2019exponential}, but their bound still has a small term that would not vanish even with an infinite sample size and a vanishing stability parameter.

In addition, researchers demonstrate that many popular optimization methods, such as SGD, satisfy algorithmic stability. In \citet{hardt2015train}, the authors show that SGD satisfies uniform stability. \citet{lei2020fine} further relax the smoothness and convexity assumptions, and others instead discuss the nonconvex case for SGD in more detail~\citep{kuzborskij2018data, madden2020high}.  \citet{kuzborskij2018data} recently propose another notion of data-dependent stability for SGD. Our work can be viewed as a relaxation of uniform stability and SGD will be shown to satisfy our new notion of algorithmic stability.

\section{Locally Elastic Stability}
\label{sec:locally-elast-stab}

We first collect some notations that are used throughout this paper, which mostly follows that of \citet{bousquet2002stability}. Denote the input by e. One instance of $\cZ$ is $\cX\times\cY$, where $\cX$ and $\cY$ are input space and label space respectively. For a function class $\cF$, a learning algorithm $\mathscr{A}: \cZ^m \rightarrow \cF$ takes the training set $S$ as input and outputs a function $\cA_S \in \cF$. For any $m$-sized training set $S$, let $S^{\backslash i}=\{z_1,\cdots,z_{i-1},z_{i+1},\cdots,z_m\}$ be derived by removing the $i$th element from $S$ and $S^i=\{z_1,\cdots,z_{i-1},z'_i,z_{i+1},\cdots,z_m\}$ be derived by replacing the $i$th element from $S$ with another example $z'_i$. For any input $z$, we consider a loss function $l(f,z)$. We are particularly interested in the loss $l(f, z)$ when the function $f = \cA_S$.


Now, we formally introduce the notion of locally elastic stability below. Let $\beta_m(\cdot, \cdot)$ be a sequence of functions indexed by $m \ge 2$ that each maps any pair of $z, z' \in \cZ$ to a positive value.
\begin{definition}[Locally Elastic Stability]\label{def:LE}
An algorithm $\sA$ has locally elastic stability $\beta_m(\cdot, \cdot)$ with respect to the loss function $l$ if, for all $m$, the inequality
\[
\left| l(\cA_S,z)-l(\cA_{S^{\backslash i}},z) \right| \le \beta_m(z_i,z)
\]
holds for all $S\in \cZ^m, 1 \le i \le m$, and $z\in \cZ$.
\end{definition}


In words, the change in the loss function due to the removal of any $z_i$ is bounded by a function depending on both $z_i$ and the data point $z$ where the loss is evaluated. In this respect, locally elastic stability is \textit{data-dependent}. In general, $\beta_m(\cdot, \cdot)$ is not necessarily symmetric with respect to its two arguments. To further appreciate this definition, we compare it with uniform stability, which is perhaps one of the most popular algorithmic stability notions.

\begin{definition}[Uniform Stability \citep{bousquet2002stability}]\label{def:uni}
\label{def:uniform-stability}
Let $\beta^{\textnormal{U}}_m$ be a sequence of scalars. An algorithm $\sA$ has uniform stability $\beta^{\textnormal{U}}_m$ with respect to the loss function $l$ if 
\begin{equation}\label{eq:u_stab}
\left| l(\cA_S,z)-l(\cA_{S^{\backslash i}},z) \right| \le \beta^{\textnormal{U}}_m
\end{equation}
holds for all $S\in \cZ^m, 1 \le i \le m$, and $z\in \cZ$.

\end{definition}



First of all, by definition one can set $\beta^{\textnormal{U}}_m = \sup_{z', z} \beta_m(z', z)$. Furthermore, a simple comparison between the two notions immediately reveals that locally elastic stability offers a finer-grained definition of the loss function sensitivity. The gain is significant particularly in the case where the worst possible value of $\left| l(\cA_S,z)-l(\cA_{S^{\backslash i}},z) \right|$ is much larger than its typical realizations. 

\subsection{Estimation Using Influence Functions}
\label{subsec:IF}
In the introduction part, we motivated the proposal of locally elastic stability by showing the class-level sensitivity for a random feature model and neural networks in Figure \ref{fig:IF-results}.
In this subsection, we elaborate more on the experimental results and the corresponding approximation method. The examples we considered demonstrate small $\beta_m(z_i,z)$ for most $z$'s in $\cZ$ for any training example $z_i$,. The fact that $\beta_m(z_i,z)$ is small for most of $z$'s is important to obtain a sharper generalization bound with locally elastic stability than the bound with uniform stability.

Specifically, consider a function $f$ that is parameterized by $\theta$ and write $l(\theta,z)$ instead of $l(f, z)$ for the loss. Writing $f = f_\theta$, the algorithm $\sA$ aims to output $f_{\hat{\theta}}$ where $\hat{\theta}=\argmin_{\theta\in\Theta}\sum_{j=1}^ml(\theta,z_j)/m$ (we temporarily ignore the issue of uniqueness of minimizers here). Then, $\cA_S$ defined previously is exactly $f_{\hat{\theta}}$. Denote $\hat{\theta}^{\backslash i}=\argmin_{\theta\in\Theta}\sum_{j\neq i}l(\theta,z_j)/m,$ we aim to quantitatively estimate $|l(\hat{\theta},z)-l(\hat{\theta}^{\backslash i},z)|$ for all $i$'s. However, quantifying the above quantity for all $i$'s is computationally prohibitive in practice for neural networks and also a pain even for random feature model. In order to alleviate the computational issue, we adopt influence functions from \citet{koh2017understanding} and consider the same simplified model as in \citet{koh2017understanding}: an $N$-layer neural network whose first $N-1$ layers are pre-trained. Given that model, when the loss function $l(\theta,z)$ is strictly convex in $\theta$ for all $z$, such as the continuously used cross entropy loss and squared loss with $l_2$ penalty on $\theta$,  we have the following approximation:
\begin{align}\label{eq:IF}
\begin{split}
\beta_m(z_i,z)&:=|l(\hat{\theta},z)-l(\hat{\theta}^{\backslash i},z)|\\
&\approx \frac{1}{m}\big|\nabla _\theta l(\hat{\theta},z)H^{-1}_{\hat{\theta}}\nabla_\theta l (\hat{\theta},z_i)\big|,
\end{split}
\end{align}
where $H_{\hat{\theta}}=\sum_{j=1}^m\nabla^2 l (\hat{\theta},z_j)/m$ is the Hessian. We remark that it is very common in transfer learning to pre-train the $N-1$ layers and it is different from the random feature model, where the first $N-1$ layers are chosen to be independent of data.  We further consider training the full $N$-layer neural  networks by analyzing the sensitivity of the loss step-wisely for SGD in the Appendix. In Figure \ref{fig:IF-results}, we demonstrate the class-level sensitivity approximated by influence functions for neural networks (based on a pre-trained $18$-layer ResNet \citep{he2016deep}) and a random feature model (based on a randomly initialized $18$-layer Resnet) on CIFAR-10 \citep{krizhevsky2009learning}.

The results indicate that for random feature models and neural networks with a training example $z_i$, if $z$ is from the same class of $z_i$, then $\beta_m(z_i,z)$ is large; if $z$ is from a different class $\beta_m(z_i,z)$ is small. Recognizing the long tail property of class frequencies for image datasets in practice, it would lead to small $\beta_m(z_i,z)$'s for most $z$'s for any training example $z_i$. 

We close this section by providing empirical evidence to justify our statement that ``$\beta_m(z_i,z)$ for most $z$'s is small for any training example $z_i$." Specifically, we compare $\sup_{z'\in S,z\in \cZ}\beta_m(z',z)$ and $\sup_{z'\in\cZ}\bE_{z}\beta_m(z',z)$ for both neural networks and the random feature model, and the results are shown in Table \ref{table:LES-vs-US}.

\begin{table}
\centering
\scalebox{0.5}{
\begin{tabular}{|c||c|c|c|}
\hline
Models & $\sup_{z'\in S,z\in \cZ}\beta_m(z',z)$ & $\sup_{z'\in\cZ}\bE_{z}\beta_m(z',z)$ & $ratio$ \\ \hline
Neural networks & $3.05$ & $0.02$ & $153$\\ \hline
Random feature model &$1.73$ & $0.04$ & $43$\\ \hline
\end{tabular}}
\caption{
Comparison between locally elastic stability and uniform stability for neural networks and the random feature model in Figure \ref{fig:IF-results}. 
}
\label{table:LES-vs-US}
\end{table}

It is worth noticing the dependence on whether the two images belong to the same class results in a pronounced diagonal structure in Figure \ref{fig:NNs-IF} and \ref{fig:RF-IF}, and in contrast, linear models do not exhibit such a strong dependence on the class of images, as evidenced by the absence of a diagonal structure in Figure~\ref{fig:Linear-IF}. We believe the above phenomenon is one of the reasons that neural networks generalize well and our new proposed stability provides a new direction towards understanding the generalization behavior of neural networks.

\subsection{Connection with Local Elasticity}

Locally elastic stability has a profound connection with a phenomenon identified by \citet{he2020local}, where the authors consider the question: how does the update of weights of neural networks using induced gradient at an image (say a tiger) impact the prediction at another image? In response to this question, \citet{he2020local} observe that the impact is significant if the two images have the same membership (e.g., the test image is another tiger) or share features (e.g., a cat), and the impact is small if the two images are not semantically related (e.g., a plane).\footnote{To be complete, this phenomenon does not appear in the initialized neural networks and become pronounced only after several epochs of training on the dataset.} In contrast, this phenomenon is generally not present in kernel methods, and \citet{chen2020label} argue that this absence is in part responsible for the ineffectiveness of neural tangent kernels compared to real-world neural networks in terms of generalization. Related observations have been made in \citet{chatterjee2020coherent} and \citet{fort2019stiffness}. This phenomenon, which \citet{he2020local} refer to as local elasticity, would imply the characteristic of neural networks that we observe in Figure~\ref{fig:IF-results}. Intuitively, from local elasticity we would expect that if we remove an image of a cat in the training set $S$, the loss after training on a test image of a plane would not be affected much compared with the loss obtained by training on the original training set (assuming the same randomness from sampling). Conversely, the final loss would be affected much if the test image is another tiger. Our Definition \ref{def:LE} formalizes the intuition of local elasticity by incorporating the membership dependence into the sensitivity of the loss function, hence is named as locally elastic stability.

The improvement brought by locally elastic stability is further enhanced by the \textit{diversity} of real-life data. First, the number of classes in tasks resembling practical applications is often very large. For example, the ImageNet dataset contains more than 1000 classes \citep{deng2009imagenet}. For most pairs $z', z$, their class memberships are different, leading to a relatively small value of $\beta_m(z', z)$ compared to the uniform upper bound adopted in uniform stability. Moreover, the long tail property of real-life images suggest that the class of cats, for example, consists of many cats with different appearances and non-vanishing frequencies \citep{zhu2014capturing} (further elaborations are included in the Appendix). Combining with the observations mentioned above, we would expect that for any fixed training example $z'$, $\beta_m(z',z)$ would be small for most  $z$ sampled from the distribution $\cD$. Therefore, the use of a uniform upper bound on the sensitivity is too pessimistic.

\section{Generalization Bounds}\label{sec:generalization}

In this section, we present our generalization bound for locally elastically stable algorithms and compare it to those implied by classical algorithmic stability notions.



\paragraph{Assumptions.} We assume the space $\cZ$ is bounded. 
In addition, from the approximation shown in (\ref{eq:IF}), for many problems one has $|l(\hat{\theta},z)-l(\hat{\theta}^{\backslash i},z)|=O(1/m)$. Moreover, \citet{bousquet2002stability} show that $\beta^{\textnormal{U}}_m$ in uniform stability satisfies $\beta^{\textnormal{U}}_m = O(1/m)$ for many problems including bounded SVM, $k$-local rules, and general regularization algorithms. This fact suggests that it is reasonable to expect that $\beta_m(z', z) = O_{z',z}(1)/m$ for locally elastic stability. More specifically, we have the following assumption.

\begin{assumption} \label{ass:1}
For the function $\beta_m(\cdot,\cdot)$, for any $z,z'\in\cZ$,
\[
\beta_m(z',z) = \frac{\beta(z',z)}{m}
\]
for some function $\beta(\cdot,\cdot)$ that is independent of $m$. In addition,  $\beta(\cdot, z)$ as a function of its first argument is $L$-Lipchitz continuous for all $z\in \cZ$ and the loss function and there exists $M_\beta>0$ such that $|\beta(\cdot,\cdot)|\le M_\beta$.
\end{assumption}

In essence,  $\beta_m(z',z) = \beta(z',z)/m$ is equivalent to assuming that $\sup_m m \beta_m(z',z)$ is finite for all $z', z$. The boundedness assumption of $\beta(\cdot,\cdot)$ holds if $\beta$ is a continuous function in conjunction with the boundedness of $\cZ$. In relating this assumption to uniform stability in Definition \ref{def:uni}, we can take $\beta_m^{\textnormal{U}} = M_{\beta}/m$.

Now, we are ready to state our main theorem. For convenience, write $\Delta(\cA_S)$ as a shorthand for the defect $\bE_z l(\cA_S, z) -  \sum_{j=1}^ml(\cA_{S},z_j)/m$, where the expectation $\bE_z$ is over the randomness embodied in $z \sim \cD$. In particular, $\bE_z l(\cA_S, z)$ depends on $\cA_S$.


%
%

\begin{theorem}\label{thm:main}
Let $\sA$ be an algorithm that has locally elastic stability $\beta_m(\cdot,\cdot)$ with respect to the loss function $l$, which satisfies $0\le l\le M_l$ for a constant $M_l$.  Under Assumption \ref{ass:1}, for any given $\eta$ and any $0 < \delta < 1$ and , for sufficiently large $m$, with probability at least $1-\delta$, we have
\begin{align*}
\Delta(\cA_S) &\le \frac{2\sup_{z'\in \cZ}\bE_{z}\beta(z',z)}{m}\\
+&2\left( 2\sup_{z'\in \cZ}\bE_{z}\beta(z',z)+\eta+M_l \right)\sqrt{\frac{2\log(2/\delta)}{m}}.
\end{align*}

\end{theorem}

We remark that (1). the parameter $\eta$ in Theorem 3.1 is used to control the deviation $\sup_{z'\in \cZ}\Big|\sum_{j\neq k}\beta(z',z_j)/m-\bE_{z}\beta(z',z)\Big|.$ As shown in our Lemma A.4 in the Appendix, we only need $\eta>2M_\beta/m$. Thus, as stated in our theorem, \textit{for any given $\eta>0$}, as long as the sample size is large enough, i.e. $m>2M_\beta/\eta$, all the claims involving $\eta$ hold. In the subsequent discussions, for instance, in Section 4.1, we can set $\eta=\sup_{z'\in \cZ}\bE_{z}\beta(z',z)$ and Theorem 3.1 still holds.  (2). Theorem~\ref{thm:main} holds for $m$ that is larger than a bound depending on $\delta, \eta, d, L, M_{\beta}$, and $M_{l}$, One coarse and sufficient condition provided in the Appendix is that $m$ is large enough such that $\log (C' m)/m \leq \eta^2/(64M^2_\beta)$, $2M^2\log(2/\delta)/(\tilde{M}^2m)\leq \eta^2/(128M^2_\beta)$, $2M/\tilde{M}\sqrt{2\log(2/\delta)/m}\leq\eta^2/(128M^2_\beta),$ and $m>2M_\beta/\eta$ for constants $C'$ (depnding on $d$), $\tilde{M}, M_\beta, \eta$, which can be achieved once we notice that $\lim_{m\rightarrow\infty}\log (C' m)/m\rightarrow 0.$

In this theorem, the bound on the defect $\Delta(\cA_S)$ tends to 0 as $m \rightarrow \infty$. The factor $\sqrt{\log(2/\delta)}$ results from the fact that this locally elastic stability-based bound is an exponential bound. Notably, the bound depends on locally elastic stability through $\sup_{z'\in \cZ}\bE_{z}\beta(z',z)$, which is closely related to error stability~\citep{kearns1999algorithmic}. See more discussion in Section~\ref{sec:comparison}.

\begin{remark}
In passing, we make a brief remark on the novelty of the proof. An important step in our proof is to take advantage of the fact that $|\sum_{j\neq k}\beta(z',z_j)/m - \bE_{z}\beta(z',z)|$ is small with high probability. Conditioning on this event, however, $z_j$'s are no longer an i.i.d.~sample from $\cD$. The dependence among input examples would unfortunately invalidate McDiarmid's inequality, which is a key technique in proving generalization bounds for uniform stability. To overcome this difficulty, we develop new techniques to obtain a more careful analysis of some estimates. More details can be found in our Appendix.
\end{remark}

\section{Comparisons with Other Notions of Algorithmic Stability}
\label{sec:comparison}
Having established the generalization bound for locally elastic stability, we compare our results with some classical notions \cite{bousquet2002stability}. As will be shown in this subsection, error stability is not sufficient to guarantee generalization, hypothesis stability only yields polynomial bounds, and uniform stability only considers the  largest loss change on $z$ in $\cZ$ by removing  $z_i$ from $S$. In contrast, locally elastic stability not only provides exponential bounds as uniform stability but also takes into account the varying sensitivity of the loss. This fine-grained perspective can be used to improve the generalization bounds derived from uniform stability, when the average loss change by removing $z_i$ from $S$ over different $z$'s in $\cZ$ is much smaller than the worst-case loss change.

\subsection{Uniform Stability} Following \citet{bousquet2002stability}, for an algorithm $\sA$ having uniform stability $\beta^U_{m}$ (see Definition \ref{def:uniform-stability}) with respect to the loss function $l$, if $0\le l(\cdot,\cdot)\le M_l$, for any $\delta\in(0,1)$ and sample size $m$,  with probability at least $1-\delta$,
\begin{align*}
\Delta(\cA_S) \le 2\beta^{\textnormal{U}}_{m}+(4m\beta^{\textnormal{U}}_m+M_l)\sqrt{\frac{\log(1/\delta)}{2m}}.
\end{align*}
Notice that if an algorithm $\sA$ satisfies locally elastic stability with $\beta_m(\cdot,\cdot)$, then it it has uniform stability with parameter $\beta^U_{m}:=\sup_{z'\in S,z\in \cZ}\beta(z',z)/m$. We can identify $\sup_{z'\in S,z\in \cZ}\beta(z',z)$ with $M_\beta$ in Assumption \ref{ass:1}.




To get a better handle on the tightness of our new generalization bound, we revisit some classic examples in \citet{bousquet2002stability} and demonstrate the superiority of using our bounds over using uniform stability bounds in certain cases. In order to have a clear presentation, let us briefly recap the assumptions and concepts used in \citet{bousquet2002stability}. 

\begin{assumption}
Any loss function $l$ considered in this paragraph is associated with a cost function $c_l$, such that for a hypothesis $f$ with respect to an example $z=(x,y)$, the loss function is defined as
$$l(f,z)=c_l(f(x),y).$$
\end{assumption}

\begin{definition}
A loss function $l$ defined on $\cY^\cX\times \cY$ is $\sigma$-admissible with respect to $\cY^\cX$ if the associated cost function $c_l$ is convex with respect to its first argument and the following condition holds: for any $y_1,y_2\in\cY$ and any $y'\in \cY$
$$|c_l(y_1,y')-c_l(y_2,y')|\le \sigma \|y_1-y_2\|_\cY,$$
where $\|\cdot\|_\cY$ is the corresponding norm on $\cY$.
\end{definition}

\paragraph{Reproducing kernel Hilbert space.} A reproducing kernel Hilbert space (RKHS) $\cH$ is a Hilbert space of functions, in which point evaluation is a continuous linear functional and satisfies for any $h\in \cH$, any $x\in\cX$
$$h(x)=\langle h,K(x,\cdot) \rangle$$
where $K$ is the corresponding kernel of $\cH$. In particular, by Cauchy-Schwarz inequality, for any $h\in\cH$, any $x\in\cX$
$$|h(x)|\le \|h\|_K\sqrt{K(x,x)},$$
where $\|\cdot\|_K$ is the norm induced by kernel $K$ for the reproducing kernel Hilbert space $\cH$. We denote $\sqrt{K(x,x)}$ as $\kappa(x)$. Notice that for the reproducing kernel Hilbert space, $K$ must be a positive semi-definite kernel and $\kappa(x)\ge 0$.

In order to derive locally elastic stability bounds, we introduce the following lemma, which is a variant of Theorem $22$ in \citet{bousquet2002stability}.

\begin{lemma}\label{lm:kh}
Let $\cH$ be a reproducing kernel Hilbert space with kernel $K$, and for any $x\in\cX$, $K(x,x)\le \kappa^2<\infty$. The loss function $l$ is $\sigma$-admissible with repect to $\cH$ and the learning algorithm is defined by
$$\cA_S=\argmin_{h\in\cH} \frac{1}{m}\sum_{j=1}^m l(h,z_j)+\lambda \|h\|^2_K,$$
\end{lemma}
where $\lambda$ is a positive constant. Then, $\cA_S$ has uniform stability $\beta^U_m$ and locally elastic stability $\beta_m(z_i,z)$ such that
$$\beta^U_m\le \frac{\sigma^2 \kappa^2}{2\lambda m}~~\text{and}~~\beta_m(z_i,z)\le  \frac{\sigma^2 \kappa(x_i)\kappa(x)}{2\lambda m}.$$

Now, we are ready to investigate how the locally elastic bounds improve over the uniform stability bounds in the bounded SVM regression and regularized least square regression studied in \citet{bousquet2002stability}. We remark here, following the same settings in \citet{bousquet2002stability}, though the algorithms in the examples below are minimizing a regularized version of the loss,  the generalization gap studied above  is still $\Delta(\cA_S)=\bE_z l(\cA_S, z) -  \sum_{j=1}^ml(\cA_{S},z_j)/m$.

\begin{example}[Stability of bounded SVM regression]\label{ex:svm}
Assume $K$ is a bounded kernel, such that $K(x,x)\le \kappa^2$ for all $x\in\cX$, and $\cY=[0,B]$ for a real positive number $B$. Consider the loss function for $\tau>0$,
$$l(f,z)=|f(x)-y|_\tau=\left\{\begin{matrix}
 0,~~~\text{if}~|f(x)-y|\le \tau, \\ 
|f(x)-y|-\tau, ~\text{otherwise}. 
\end{matrix}\right.$$

The learning algorithm is defined by
$$\cA_S=\argmin_{h\in\cH} \frac{1}{m}\sum_{j=1}^m l(h,z_j)+\lambda \|h\|^2_K.$$

Noting $0\le l(f,z)\le B$ and $\sigma=1$ in our case \footnote{There are several small typos in the original Example $1$ and $3$ in \citet{bousquet2002stability} with respect to the range of $l(f,z)$ which we correct in our examples. } and using Lemma \ref{lm:kh}, we obtain the following bound via uniform stability:
\begin{align*}
\Delta(\cA_S) \le \frac{\kappa^2}{\lambda m}+\left(\frac{2\kappa^2}{\lambda }+B\right)\sqrt{\frac{\log(1/\delta)}{2m}}.\quad(\cB_1)
\end{align*}
In addition, we obtain the following bound via locally elastic stability by choosing $\eta=\kappa\bE_x\kappa(x)/\lambda$
\begin{align*}
&\Delta(\cA_S) \le \frac{\kappa\bE_x\kappa(x)}{\lambda m}\\
&+\left(\frac{3\kappa\bE_x\kappa(x)}{\lambda }+2B\right)\sqrt{\frac{2\log(2/\delta)}{m}}.\quad(\cB_2)
\end{align*}

For simplicity, we consider $K$ to be the bilinear kernel (similar analysis can be extended to other kernels such as polynomial kernels) $K(x,x')=\langle x,x' \rangle$ and all $x\in\cX$'s norm are bounded by $B'$. Then, $\kappa= B'^2$ and $\bE_x\kappa(x)=\bE_x\|x\|^2$. Apparently, the first term on the RHS in $(\cB_2)$ is smaller than the first term on the RHS in $(\cB_1)$. So we focus on comparing the second terms for both inequalities. For $\delta<0.5$, we have $\log(2/\delta)\le 2\log(1/\delta)$. Applying the above inequality to $(\cB_2)$, we can simplify the expressions, and   if we further have
\begin{equation}\label{eq:whatever}
\left(\frac{2\kappa^2}{\lambda }+B\right)\ge 2\sqrt{2}\left(\frac{3\kappa\bE_x\kappa(x)}{\lambda }+2B\right)
\end{equation}

the bound obtained in $(\cB_2)$ is tighter than the one in $(\cB_1)$. If the scale of $\kappa\bE_x\kappa(x)/\lambda$ and $B$ and are relatively small comparing with $\kappa^2/\lambda$, (\ref{eq:whatever}) apparently holds. Notice the first requirement regarding $\kappa\bE_x\kappa(x)/\lambda$ being relatively small comparing with $\kappa^2/\lambda$ is distribution-dependent and can be easily achieved if the distribution of $\|x\|$ is concentrated around zero. If we further have $B^{'2}$ is large enough comparing with $B\lambda$, the bound in $(\cB_2)$ is tighter than the one in $(\cB_1)$.

In particular, if $x$ is a distribution such that 
$$\cP\left(\|x\|\le \frac{B'}{6}\right)\ge \frac{23}{24},$$
and $B'^2\ge 8\sqrt{2}B\lambda$, $\delta<0.5$, the bound obtained in $(\cB_2)$ is tighter than the one in $(\cB_1)$. Moreover, our locally elastic bound is \textbf{significantly tighter} than the one obtain via uniform stability, if $B'^2\gg B\lambda$.
\end{example}

\begin{example}[Stability of regularized least square regression] \label{ex:reg}
Consider $\cY=[0,B]$ and denote $\cH$ as the reproducing kernel Hilbert space induced by kernel $K$. The regularized least square regression algorithm is defined by 
$$\cA_S=\argmin_{h\in\cH} \frac{1}{m}\sum_{j=1}^m l(h,z_j)+\lambda \|h\|^2_K,$$
where $l(f,z)=(f(x)-y)^2$. Then, with Lemma \ref{lm:kh}, we obtain the following bound via uniform stability
\begin{align*}
\Delta(\cA_S) \le \frac{4\kappa^2B^2}{\lambda m}+\left(\frac{8\kappa^2B^2}{\lambda }+B^2\right)\sqrt{\frac{\log(1/\delta)}{2m}}.~~(\cB_3)
\end{align*}
Meanwhile, we can obtain the following bound via locally elastic stability
\begin{align*}
&\Delta(\cA_S) \le \frac{4\kappa\bE_x\kappa(x)B^2}{\lambda m}\\
&+\left(\frac{12\kappa\bE_x\kappa(x)B^2}{\lambda }+2B^2\right)\sqrt{\frac{2\log(2/\delta)}{m}}.\quad(\cB_4)
\end{align*}

Similarly, for simplicity, let us consider $K$ to be the bilinear kernel $K(x,x')=\langle x,x' \rangle$ and all $x\in\cX$'s norm are bounded by $B'$. $\kappa= B'^2$ and $\bE_x\kappa(x)=\bE_x\|x\|^2$. With the same spirit as in Example \ref{ex:svm}, if $x$ is a distribution such that 
$$\cP\left(\|x\|\le \frac{B'}{4}\right)\ge \frac{3}{4}$$
and suppose $B'^4\ge \lambda$, $\delta<0.5$, the bound obtained in $(\cB_4)$ is tighter than the one in $(\cB_3)$. Similar as Example \ref{ex:svm}, our locally elastic bound is significantly tighter than the one obtain via uniform stability, if $B'^4\gg \lambda$.

\end{example}

\subsection{Hypothesis Stability.} 
%
For a training set $S$ with $m$ examples, an algorithm $\sA$ has hypothesis stability $\beta^H_m$ with respect to the loss function $l$ if
$$\bE_{S,z}\left|l(\cA_S,z)-l(\cA_{S^{\backslash i}},z) \right|\le \beta^H_{ m}$$ 
holds for all $S\in \cZ^m$ and $1\le i\le m$, where $\beta^{\textnormal{H}}_m$ is a sequence of scalars . If $0\le l(\cdot,\cdot)\le M_l$, for any $\delta\in(0,1)$ and sample size $m$, \citet{bousquet2002stability} show that with probability at least $1-\delta$
\[
\Delta(\cA_S) \le \sqrt{\frac{M_l^2+12M_lm\beta^H_m}{2m\delta}}.
\]
For $\beta^H_{m}=O(1/m)$, hypothesis stability only provides a tail bound of order $O(1/\sqrt{m\delta})$ (polynomial tail bound) while locally elastic stability provides tail bounds of order $O(\sqrt{\log(1/\delta)/m})$ (exponential tail bound). In addition, for an algorithm $\sA$ satisfying locally elastic stability $\beta_m(\cdot,\cdot)$, it by definition satisfies hypothesis stability with parameter  $\bE_{z',z}\beta_m(z',z)$.

\subsection{Error Stability.} For a training set $S$ with $m$ examples, an algorithm $\sA$ has error stability $\beta^E_m$ with respect to the loss function $l$ if
$$\left|\bE_{z}[l(\cA_S,z)]-\bE_{z}[l(\cA_{S^{\backslash i}},z)] \right|\le \beta^E_{m},$$ 
for all $S\in \cZ^m$ and $1\le i\le m$. Error stability is closely related to locally elastic stability in the sense that $\beta^E_m$ can take the value of $\sup_{z'\in\cZ}\bE_{z}\beta_m(z',z)$. However, as pointed out by \citet{kutin2012almost}, this notion is too weak to guarantee generalization in the sense that there exists an algorithm $\sA$ where the error stability parameter goes to $0$ as the sample size $m$ tends to infinity but the generalization gap does not go to $0$.

\begin{figure*}[t]
		\centering
		 \subfigure[Neural networks (epoch $0$).]{
			\centering
			\includegraphics[scale=0.33]{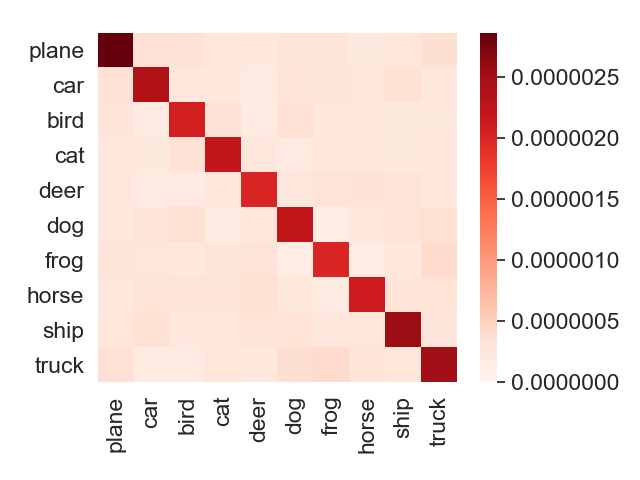}
			\label{fig:NNs-SGD-0}}
        \subfigure[Neural networks (epoch $10$).]{
			\centering
			\includegraphics[scale=0.33]{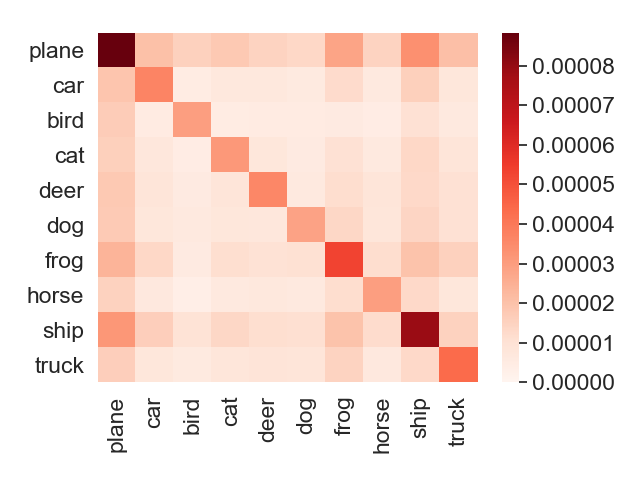}
			\label{fig:NNs-SGD-10}}
		\subfigure[Neural networks (epoch $50$).]{
			\centering
			\includegraphics[scale=0.33]{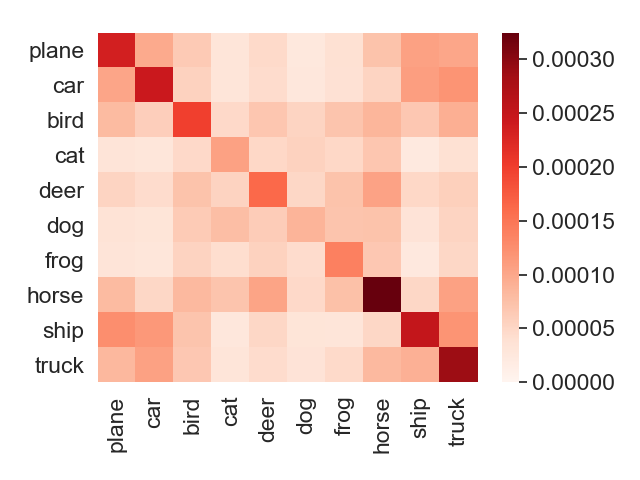}
			\label{fig:NNs-SGD-50}} \\
		\subfigure[Random feature model (epoch $0$).]{
			\centering
			\includegraphics[scale=0.33]{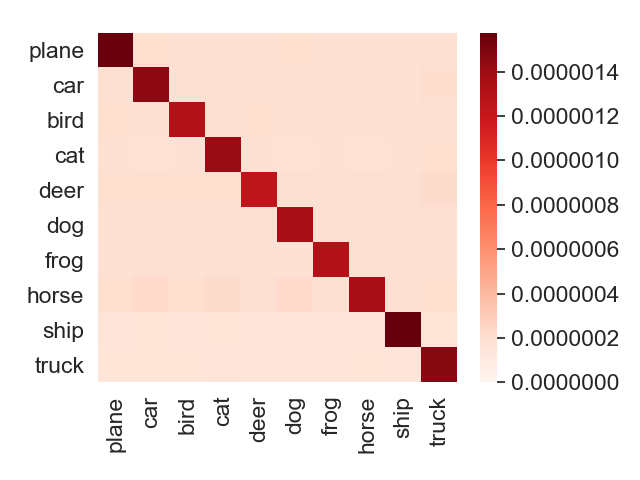}
			\label{fig:RF-SGD-0}}
        \subfigure[Random feature model (epoch $50$).]{
			\centering
			\includegraphics[scale=0.33]{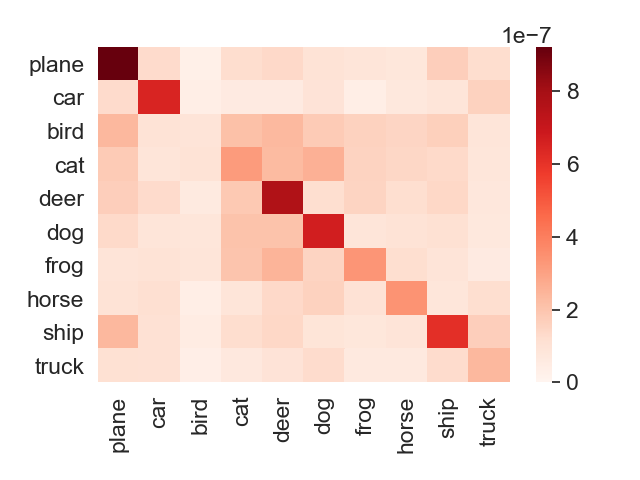}
			\label{fig:RF-SGD-50}}
		\subfigure[Random feature model (epoch $250$).]{
			\centering
			\includegraphics[scale=0.33]{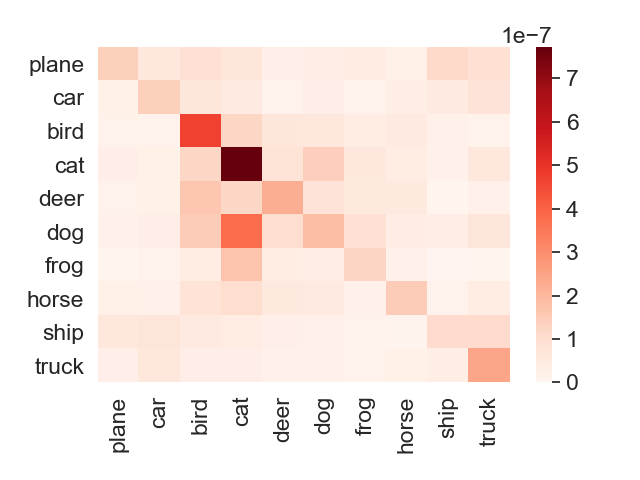}
			\label{fig:RF-SGD-250}} 
		\caption{Exact stepwise characterization of class-level sensitivity for neural networks and random feature models trained with different numbers of epochs by SGD on CIFAR-10. The class-level sensitivity for a stepwise update of SGD is $C'(c_a, c_b) = \frac{1}{|S_a|\cdot |\tilde{S}_b|}\sum_{z_i \in S_a} \sum_{z \in \tilde{S}_b}|l(\htheta_t-\eta\nabla_\theta l(\htheta_t,z_i),z)-l(\htheta_t,z)|$, where $S_a$ denotes the set of examples with class $a$ in the training data and $\tilde{S}_b$ denotes the set of examples with class $b$ in the test data. 
		}
		\label{fig:SGD-results}
\end{figure*}

\section{Locally Elastic Stability and Stochastic Gradient Descent}
In \citet{hardt2016train}, the authors demonstrate that SGD satisfies uniform stability under the standard Lipschitz and smoothness assumptions. As another concrete application of locally elastic stability, we revisit this problem and show that SGD also satisfies locally elastic stability under similar assumptions.

SGD algorithm consists of multiples steps of stochastic gradient updates $\htheta_{t+1}=\htheta_t-\eta_t\nabla_{\theta}l(\htheta_t,z_{i_t}),$ where we allow the learning rate to change over time and $\eta_t$ is the learning rate at time $t$, $i_t$ is picked uniformly at random from $\{1,\cdots,m\}$. Throughout this subsection, we develop our results for a $T$-step SGD. For a randomized algorithm $\sA$ like SGD, we can extend the definition of locally elastic stability just as \citet{hardt2016train} do for uniform stability (Definition \ref{def:uniform-stability}). As shown in Figure \ref{fig:SGD-results}, we further demonstrate the step-wise characterization of class-level sensitivity for neural networks (based on a pre-trained ResNet-18) and random feature models (based on a randomly initialized ResNet-18) trained for different numbers of epochs by SGD on CIFAR-10. 
\begin{definition}
A randomized algorithm $\sA$ is $\beta_m(\cdot,\cdot)$-locally elastic stable if for all datasets $S\in\cZ^n$, we have 
\begin{equation*}
|\bE_{\sA}[l(\cA_S,z)]-\bE_{\sA}[l(\cA_{S^{\backslash i}},z)]|\le\beta_m(z_i,z),
\end{equation*}
where the expectation is over the randomness embedded in the algorithm $\sA$. 
\end{definition}
For SGD, the algorithm $\sA$ outputs functions $\cA_S$ and $\cA_{S^{\backslash i}}$ which are parameterized by $\hat{\theta}_T$ and $\hat{\theta}^{\backslash i}_T$ and we further study whether there is a function $\beta_m(\cdot,\cdot)$ such that
$|\bE[l(\hat{\theta}_T,z)]-\bE[l(\hat{\theta}^{\backslash i}_T,z)]|\le\beta_m(z_i,z),$
where the expectation is taken with respect to randomness coming from uniformly choosing the index at each iteration. Under similar settings as in \citet{hardt2016train}, we develop estimates of the locally elastic stability parameters separately for convex, strongly convex, and non-convex cases. Due to the space constraint, we only show our results here for convex and non-convex cases and defer the treatment of  the strongly convex case to Appendix.

\begin{proposition}[Convex Optimization]
Assume that the loss function $l(\cdot, z)$ is $\alpha$-smooth and convex for all $z\in\cZ$. In addition, $l(\cdot, z)$ is $L(z)$-Lipschitz and $L(z)<\infty$ for all $z\in\cZ$: $|l(\theta,z)-l(\theta',z)|\le L(z)\|\theta-\theta'\|$ for all $\theta, \theta'.$ We further assume $L=\sup_{z\in\cZ}L(z)<\infty$. Suppose that we run SGD with step sizes $\eta_t\le 2/\alpha$ for $T$ steps. Then, 
\begin{equation*}
|\bE[l(\hat{\theta}_T,z)]-\bE[l(\hat{\theta}^{\backslash i}_T,z)]|\le\frac{(L+L(z_i))L(z)}{m}\sum_{t=1}^T\eta_t.
\end{equation*}
 \end{proposition}

\begin{proposition}[Non-convex Optimization]\label{prop:nonconvex}
Assume that the loss function $l(\cdot, z)$ is non-negative and bounded for all $z\in\cZ$. Without loss of generality, we assume $0 \le l(\cdot, z)\le 1$. In addition, we assume $l(\cdot, z)$ is $\alpha$-smooth. We further assume $l(\cdot, z)$ is $L(z)$-Lipschitz and $L(z)<\infty$ for all $z\in\cZ$ and $L=\sup_{z\in\cZ}L(z)<\infty$. Suppose that we run SGD for $T$ steps with monotonically non-increasing learning rate $\eta_t\le c/t$ for some constant $c>0$. Then, 
\begin{align*}
|\bE[l(\hat{\theta}_T,z)]-\bE[l(\hat{\theta}^{\backslash i}_T,z)]|\le \gamma_m \phi_\alpha(m,T,z_i,z),
\end{align*}
where $\phi_\alpha(m,T,z_i,z)=(c(L(z_i)+L)L(z)T^{\alpha c})^{\frac{1}{\alpha c+1}}$ and $\gamma_m =(1+1/(\alpha c))/(m-1)$. 
\end{proposition}

From the propositions above, we see that SGD has locally elastic stability with parameter taking the form $\beta(\cdot,\cdot)/m$, where $\beta(\cdot,\cdot)$ is independent of $m$. This is consistent with our assumptions regarding the form of $\beta_m(\cdot,\cdot)$ in Section \ref{sec:generalization}. We remark that unlike \citet{hardt2016train}, our results use $\hat{\theta}^{\backslash i}_T$ instead of $\hat{\theta}^{i}_T$ in order to be consistent with our definition in Definition \ref{def:LE}, where $\hat{\theta}^{i}_T$ is the parameter obtained by training on $S^i$ (replacing the $i$th element from $S$ with another example instead of removing the $i$th element as in $S^{\backslash i}$). This setting requires us to provide new techniques. Specifically, we construct new coupling sequences to obtain an upper bound on $|\bE[l(\hat{\theta}_T,z)]-\bE[l(\hat{\theta}^{\backslash i}_T,z)]|$ (see more in the Appendix).

\paragraph{Comparison with results in \citet{hardt2016train}}
By using $L(z)$ instead of $L$, if $\bE L(z)\ll L$, which holds for most common models in practice, we would expect to obtain a sharper generalization bound for SGD compared with the one derived using uniform stability in \citet{hardt2016train} according to the discussion in Section \ref{sec:comparison}. Due to limited space, let us only compare Proposition \ref{prop:nonconvex} with Theorem 3.12 in \citet{hardt2015train} with a simple example as an illustration. With some abuse of notation, we still use $\Delta(\cA_S)$ to denote $\bE_z\bE[l(\hat{\theta}_T,z)]-\sum_{i=1}^m\bE[l(\hat{\theta}^{\backslash i}_T,z)]/m$. 

In Theorem 3.12 in \citet{hardt2015train}, via uniform stability, under the assumptions in Proposition \ref{prop:nonconvex}, one can obtain the following bound:
$$\Delta(\cA_S)\le2\beta^{\textnormal{U}}_{m}+(4m\beta^{\textnormal{U}}_m+1)\sqrt{\frac{\log(1/\delta)}{2m}}, \quad(\cB_5)$$
where
$$\beta^U_m= \frac{1+1/(\alpha c)}{m-1}\left(2cL^2T^{\alpha c}\right)^{\frac{1}{\alpha c+1}}.$$
While via locally elastic stabilty, one can obtain
\begin{align*}
&\Delta(\cA_S) \le \frac{2\sup_{z'\in \cZ}\bE_{z}\beta(z',z)}{m}\\
+&2\left( 2\sup_{z'\in \cZ}\bE_{z}\beta(z',z)+1 \right)\sqrt{\frac{2\log(2/\delta)}{m}}. \quad(\cB_6)
\end{align*}
where 
$$\beta(z',z)=\frac{1+1/(\alpha c)}{m-1}\left(c(L(z')+L\right)L(z)T^{\alpha c})^{\frac{1}{\alpha c+1}}.$$

\begin{example}
Let us take $l(\theta,z)=z^2e^{-\theta^2}$, where $\theta$ and $z$ are all scalars, where $z\in [0,1]$ and $\theta\in\bR$. Apparently, the loss function is $\alpha=2$-smooth with respect to $z$. Meanwhile,
$$\frac{d}{d\theta} l(\theta,z)=-2z^2\theta e^{-\theta^2}. $$
\end{example}
The fact that $\theta e^{-\theta^2}\le e^{-1/2}/\sqrt{2}$ leads to $L(z)=\sqrt{2}e^{-1/2}z^2$ and $L=\sqrt{2}e^{-1/2}$.

With the same spirit as in Example \ref{ex:svm}, if we choose learning rate $\eta\le 1/t$, when $\delta<0.5$, as long as $\sup_{z'\in \cZ}\bE_{z}\beta(z',z)<\sqrt{2}/8\beta^U_m$ and $\beta^U_m\ge 2\sqrt{2}-1/2$, the bound obtained in $(\cB_6)$ is tighter than the one in $(\cB_5)$. It is easy to see that these conditions can be easily satisfied if $T$ is large enough and $\bE[(L(z))^{\frac{1}{3}}]<L^{\frac{1}{3}}.$
Therefore, if we further have the condition that $z$ lies in a small vicinity of $0$ with high probability (for example, $\bP\left(|z|\le \frac{1}{2}\right)>\frac{2}{3}$), then the bound obtained in $(\cB_6)$ would be tighter than the one in $(\cB_5)$. In particular, if the training time $T$ is long enough, the bound obtained in $(\cB_6)$ would be significantly tighter than the one in $(\cB_5)$.

\section{Conclusion and Future Work}
In this work, we introduce a new notion of algorithmic stability, which is a relaxation of uniform stability yes still gives rise to exponential generalization bounds. It also provides a promising direction to obtain useful theoretical bounds for demystifying the generalization ability of modern neural networks through the lens of local elasticity~\citep{he2020local}. However, as shown in Theorem \ref{thm:main}, we currently still require the sample size $m$ to be large enough so that our theoretical results hold. Whether that requirement could be removed is worthy of further investigation. In addition, our bound is related to the constant $M_l$, which is typically very large in practice if we apply the bound to neural networks. Thus, an interesting question is to examine whether this constant could be improved or not.


\section*{Acknowledgements}
We are grateful to Cynthia Dwork and Vitaly Feldman for inspiring discussions and constructive
comments. This work was supported in part by NSF through CAREER DMS-1847415, CCF-1763665 and CCF-1934876, an Alfred Sloan Research Fellowship, the Wharton Dean's Research Fund, and Contract FA8750-19-2-0201 with the US Defense Advanced Research Projects Agency (DARPA).


\bibliography{cite_LE}
\bibliographystyle{iclr2021_conference}


%

%

\onecolumn
\appendix

\noindent\textbf{\Large Appendix}
\section{Technical Details}
\label{sec:technique}
In this section, we provide the detailed proofs for our results. Let us denote 
$$R_{emp}=\frac{1}{m}\sum_{j=1}^ml(\cA_S,z_j),~R^{\backslash i}_{emp}=\frac{1}{m}\sum_{j=1}^ml(\cA_{S^{\backslash i}},z_j)$$

To prove Theorem \ref{thm:main}, we require the following key lemma.
\begin{lemma}\label{lm:alt}
Suppose an algorithm $\sA$ satisfies locally elastic stability with $\beta_m(\cdot,\cdot)$ for loss function $l$. For any $\eta>0$, let $M=2( M_\beta+\sup_{z\in \cZ}\bE_{z_j}\beta(z,z_j)+M_l)$ and $\tilde{M}= 2(2 \sup_{z\in \cZ}\bE_{z_j}\beta(z,z_j)+\eta+M_l)$. There exists a constant $C'>0$ depending on the Lipchitz constant $L$ and dimension $d$ of $z$, 
if $m$ is large enough and $\varepsilon$ is small enough, such that
$$\frac{\eta^2}{32M^2_\beta}-\frac{\log C' m}{m}\geq \frac{\varepsilon}{2\tilde{M}^2}(-\varepsilon+\frac{4\varepsilon M^2}{\tilde{M}^2}+4M),$$
we have
\begin{align*}
\bP\big(\bE_z[l(\cA_S, z)]\geq \frac{1}{m}\sum_{j=1}^ml(\cA_{S},z_j)+\frac{2\sup_{z\in \cZ}\bE_{z_j}\beta(z,z_j)}{m}+\varepsilon\big)\leq 2\exp\big(-\frac{m\varepsilon^2}{2\tilde{M}^2}\big).
\end{align*}
\end{lemma}

\begin{theorem}[Restatement of Theorem \ref{thm:main}]
Let $\sA$ be an algorithm that has locally elastic stability $\beta_m(\cdot,\cdot)$ with respect to the loss function $l$. Fixing $0 < \delta < 1$ and $\eta > 0$, for large enough $m$, with probability at least $1-\delta$, we have
\begin{align*}
\Delta(\cA_S) \le \frac{2\sup_{z'\in \cZ}\bE_{z}\beta(z',z)}{m}+2\left( 2\sup_{z'\in \cZ}\bE_{z}\beta(z',z)+\eta+M_l \right)\sqrt{\frac{2\log(2/\delta)}{m}}.
\end{align*}
\end{theorem}
\begin{proof}
Let $\delta=2\exp(-m\varepsilon^2/(2\tilde{M}^2))$, which gives us 
$$\varepsilon=\tilde{M}\sqrt{\frac{2\log(2/\delta)}{m}}.$$
Plugging the value of $\varepsilon$ into the inequality  
$$\frac{\eta^2}{32M^2_\beta}-\frac{\log C' m}{m}\geq \frac{\varepsilon}{2\tilde{M}^2}(-\varepsilon+\frac{4\varepsilon M^2}{\tilde{M}^2}+4M),$$
we obtain that 
$$\frac{\eta^2}{32M^2_\beta}-\frac{\log C' m}{m}\geq \frac{1}{2\tilde{M}}\sqrt{\frac{2\log(2/\delta)}{m}}(-\tilde{M}\sqrt{\frac{2\log(2/\delta)}{m}}+\sqrt{\frac{2\log(2/\delta)}{m}}\frac{4 M^2}{\tilde{M}}+4M).$$
It is sufficient if we have 
$$\frac{\eta^2}{32M^2_\beta}-\frac{\log C' m}{m}\geq \frac{2M^2}{\tilde{M}^2}\frac{\log(2/\delta)}{m}+\frac{2M}{\tilde{M}}\sqrt{\frac{2\log(2/\delta)}{m}}.$$

For simplicity, we let $m$ large enough such that 
$$\frac{\log C' m}{m}\leq \frac{\eta^2}{64M^2_\beta},\frac{2M^2\log(2/\delta)}{~\tilde{M}^2m}\leq \frac{\eta^2}{128M^2_\beta},~\frac{2M}{\tilde{M}}\sqrt{\frac{2\log(2/\delta)}{m}}\leq\frac{\eta^2}{128M^2_\beta},$$
which could be achieved once we notice that 
$$\lim_{m\rightarrow\infty}\frac{\log C' m}{m}\rightarrow 0.$$
Then, we obtain the desirable results by applying Lemma \ref{lm:alt}.
\end{proof}

To finish the proof of Theorem \ref{thm:main}, the only thing left is to prove Lemma \ref{lm:alt}. We prove it in the following subsection.
\subsection{Proof of Lemma \ref{lm:alt}}
\label{subsec:proof-theorem-1}
By locally elastic stability,
$$|R_{emp}-R^{\backslash i}_{emp}|\leq \frac{1}{m}\sum_{j\neq i}\frac{\beta(z_i,z_j)}{m}+\frac{M_l}{m}.$$

Recall $S=\{z_1,z_2,\cdots,z_m\}$ and we denote 
$$L(S)=\sum_{j=1}^ml(\cA_{S},z_j),~L(S^{\backslash i})=\sum_{j\neq i}^ml(\cA_{S^{\backslash i}},z_j).$$ 
Let $\cF_k$ be the $\sigma$-field generated by $z_1,\cdots,z_k$. We construct Doob's martingale and consider the associated martingale difference sequence
\begin{equation}
D_k=\bE[L(S)|\cF_k]-\bE[L(S)|\cF_{k-1}].
\end{equation}
Consider event
$$E_{\backslash k}=\Big\{S\Big|~\sup_{z'\in \cZ}\Big|\sum_{j\neq k}\frac{\beta(z',z_j)}{m}-\bE_{z}\beta(z',z)\Big|\leq\eta\Big\},$$
where $z$ is drawn from the same distribution as the training examples  $\{z_i\}_{i=1}^m$. Let us decompose $D_k$ as $D^{(1)}_k+D^{(2)}_k$, where 
$$D^{(1)}_k=\bE[L(S)I_{E_{\backslash k}}|\cF_k]-\bE[ L(S)I_{E_{\backslash k}}|\cF_{k-1}],~~D^{(2)}_k=\bE[L(S)I_{E^c_{\backslash k}}|\cF_k]-\bE[ L(S)I_{E^c_{\backslash k}}|\cF_{k-1}].$$
By Jensen's inequality,
\begin{align*}
\bE[e^{\lambda(\sum_{k=1}^mD_k)}]&\leq \frac{1}{2}\bE[e^{2\lambda(\sum_{k=1}^{m}D^{(1)}_k)}]+\frac{1}{2}\bE[e^{2\lambda(\sum_{k=1}^{m}D^{(2)}_k)}]
\end{align*}
Now, let us bound the two terms $\bE[e^{2\lambda(\sum_{k=1}^{m}D^{(1)}_k)}]$ and $\bE[e^{2\lambda(\sum_{k=1}^{m}D^{(2)}_k)}]$ separately in the following paragraphs, so as to further apply Chernoff bound to obtain a concentration bound for $\sum_{k=1}^mD_k$.
\paragraph{Bounding $\bE[e^{2\lambda(\sum_{k=1}^{m}D^{(2)}_k)}]$.}

First, we consider bounding $\bE[e^{2\lambda(\sum_{k=1}^{m}D^{(2)}_k)}]$. Let us further define 
$$A^{(2)}_k=\inf_x \bE[L(S)I_{E^c_{\backslash k}}|z_1,\cdots,z_{k-1},z_k=x]-\bE[L(S)I_{E^c_{\backslash k}}|z_1,\cdots,z_{k-1}],$$
$$B^{(2)}_k=\sup_x \bE[L(S)I_{E^c_{\backslash k}}|z_1,\cdots,z_{k-1},z_k=x]-\bE[L(S)I_{E^c_{\backslash k}}|z_1,\cdots,z_{k-1}].$$
Apparently, 
$$A^{(2)}_k\leq D^{(2)}_k\leq B^{(2)}_k. $$
Next, we provide an upper bound for $B^{(2)}_k-A^{(2)}_k$. 
Consider
\begin{align*}
B^{(2)}_k-A^{(2)}_k&=\sup_x \bE[L(S)I_{E^c_{\backslash k}}|z_1,\cdots,z_{k-1},z_k=x]-\inf_x \bE[L(S)I_{E^c_{\backslash k}}|z_1,\cdots,z_{k-1},z_k=x]\\
&\leq \sup_{x,y} \bE[L(S)I_{E^c_{\backslash k}}|z_1,\cdots,z_{k-1},z_k=x]- \bE[L(S)I_{E^c_{\backslash k}}|z_1,\cdots,z_{k-1},z_k=y]\\
&=\sup_{x,y} \bE[L(S)I_{E^c_{\backslash k}}|z_1,\cdots,z_{k-1},z_k=x]-\bE[L(S^{\backslash k})I_{E^c_{\backslash k}}|z_1,\cdots,z_{k-1},z_k=x]\\
&+\bE[L(S^{\backslash k})I_{E^c_{\backslash k}}|z_1,\cdots,z_{k-1},z_k=x]-\bE[L(S^{\backslash k})I_{E^c_{\backslash k}}|z_1,\cdots,z_{k-1},z_k=y]\\
&+\bE[L(S^{\backslash k})I_{E^c_{\backslash k}}|z_1,\cdots,z_{k-1},z_k=y]- \bE[L(S)I_{E^c_{\backslash k}}|z_1,\cdots,z_{k-1},z_k=y].
\end{align*}

By the boundedness conditions that $|\beta(\cdot,\cdot)|\leq M_\beta$, $0\leq l(\cdot,\cdot)\leq M_l$
\begin{align*}
&\bE[L(S)I_{E^c_{\backslash k}}-L(S^{\backslash k})I_{E^c_{\backslash k}}|z_1,\cdots,z_{k-1},z_k=x]+\bE[L(S^{\backslash k})I_{E^c_{\backslash k}}-L(S)I_{E^c_{\backslash k}}|z_1,\cdots,z_{k-1},z_k=y]\\
&\leq( 2M_\beta+M_l)\bP(E^c_{\backslash k}|z_1,\cdots,z_{k-1}).
\end{align*}

In addition, 
$$\bE[L(S^{\backslash k})I_{E^c_{\backslash k}}|z_1,\cdots,z_{k-1},z_k=x]-\bE[L(S^{\backslash k})I_{E^c_{\backslash k}}|z_1,\cdots,z_{k-1},z_k=y]=0.$$
As a result,
$$B^{(2)}_k-A^{(2)}_k\leq (2 M_\beta+M_l)\bP(E^c_{\backslash k}|z_1,\cdots,z_{k-1})$$
We further use $M$ to denote $2M_\beta+M_l$ and $P_k(z_{1:k-1})$ to denote $\bP(E^c_{\backslash k}|z_1,\cdots,z_{k-1})$.
Now, by Hoeffding's lemma,
\begin{align*}
\bE[e^{2\lambda(\sum_{k=1}^{m}D^{(2)}_k)}]&=\bE\Big[e^{2\lambda(\sum_{k=1}^{m-1}D^{(2)}_k)}\bE[e^{2\lambda D^{(2)}_m}|\cF_{m-1}]\Big]\\
&\leq\bE\Big[e^{2\lambda(\sum_{k=1}^{m-1}D^{(2)}_k)}e^{\frac{1}{2}\lambda^2M^2P^2_m(z_{1:m-1})}\Big]
\end{align*}
Suppose for some constant $\tau_m$ (see Lemma A.3 for exact value of  $\tau_m$)
$$\sup_k\bP(E^c_{\backslash k})\leq \tau_m,$$
then for all $k=1,\ldots,m$
$$\bP(P_k(z_{1:k-1})\geq c)\leq \frac{\tau_m}{c}.$$
Then
\begin{align*}
\bE\Big[e^{2\lambda(\sum_{k=1}^{m-1}D^{(2)}_k)}e^{\frac{1}{2}\lambda^2M^2P_m(z_{1:m-1})}\Big]&=\bE\Big[e^{2\lambda(\sum_{k=1}^{m-1}D^{(2)}_k)}e^{\frac{1}{2}\lambda^2M^2P^2_m(z_{1:m-1})}I_{\{P_m(z_{1:m-1})\geq c\}}\Big]\\
&+\bE\Big[e^{2\lambda(\sum_{k=1}^{m-1}D^{(2)}_k)}e^{\frac{1}{2}\lambda^2M^2P^2_m(z_{1:m-1})}I_{\{P_m(z_{1:m-1})< c\}}\Big]\\
&\leq \bE\Big[e^{2\lambda(\sum_{k=1}^{m-1}D^{(2)}_k)}e^{\frac{1}{2}\lambda^2M^2}I_{\{P_m(z_{1:m-1})\geq c\}}\Big]\\
&+\bE\Big[e^{2\lambda(\sum_{k=1}^{m-1}D^{(2)}_k)}e^{\frac{1}{2}\lambda^2M^2c^2}I_{\{P_m(z_{1:m-1})< c\}}\Big].
\end{align*}

Now we further bound the two terms on the righthand side of the above inequality with the following lemmas.

We first consider bounding  $\bE\Big[e^{2\lambda(\sum_{k=1}^{m-1}D^{(2)}_k)}e^{\frac{1}{2}\lambda^2M^2}I_{\{P_m(z_{1:m-1})\geq c\}}\Big]$.

\begin{lemma}\label{lm:2}
For any fixed $\lambda>0$, for any $k=1,\ldots,m$, we have
\begin{align*}
\bE[e^{2\lambda \sum_{i=1}^{k-1}D^{(2)}_{i}}I_{\{P_k(z_{1:k-1})\geq c\}}]\leq e^{2M(k-1)\lambda}\bP(P_k(z_{1:k-1})\geq c)\leq e^{2M(k-1)\lambda}\frac{\tau_m}{c}.
\end{align*}
\end{lemma}
\begin{proof}
Consider
\begin{align*}
\bE[e^{2\lambda D^{(2)}_{m-1}}I_{\{P_m(z_{1:m-1})\geq c\}}|\cF_{m-2}]&= \bE[e^{2\lambda (D^{(2)}_{m-1}-\bE [\tilde{D}^{(2)}_{m-1}|\cF_{m-2}])}I_{\{P_m(z_{1:m-1})\geq c\}}|\cF_{m-2}]\\
&\leq \bE[e^{2\lambda (D^{(2)}_{m-1}-\tilde{D}^{(2)}_{m-2})}I_{\{P_m(z_{1:m-1})\geq c\}}|\cF_{m-2}]~~(\star)\\
&\leq e^{2\lambda M}\bE_{z_{m-1}}[I_{\{P_m(\cF_{m-2},z_{m-1})\geq c \}}]
\end{align*}
In $(\star)$, $\tilde{D}^{(2)}_{m-1}|\cF_{m-2}$ is an independent copy of  $D^{(2)}_{m-1}|\cF_{m-2}$, and the inequality is due to the application of Jensen's inequality. Notice that
$$\bE\Big[e^{2\lambda D^{(2)}_{m-2}}\bE_{z_{m-1}}[I_{\{P_m(\cF_{m-2},z_{m-1})\geq c \}}]|\cF_{m-3}\Big]\leq  e^{2\lambda M}\bE_{z_{m-2},z_{m-1}}[I_{\{P_m(\cF_{m-3},z_{m-2},z_{m-1})\geq c \}}].$$
Iteratively, we obtain that 
$$\bE[e^{2\lambda \sum_{i=1}^{k-1}D^{(2)}_{i}}I_{\{P_k(z_{1:k-1})\geq c\}}]\leq e^{2M(k-1)\lambda}\bP(P_k(z_{1:k-1})\geq c).$$
Similar argument can be obtained for any $k=1,\ldots,m$.
\end{proof}
With the help of Lemma \ref{lm:2}, we can obtain an upper bound for $\bE\Big[e^{2\lambda(\sum_{k=1}^{m-1}D^{(2)}_k)}e^{\frac{1}{2}\lambda^2M^2P_m(z_{1:m-1})}\Big]$.
\begin{lemma}\label{lm:3}
For any fixed $\lambda>0$, $c<1$,
\begin{align*}
\bE\big[e^{2\lambda(\sum_{k=1}^{m-1}D^{(2)}_k)}\big]\leq e^{2m\lambda^2M^2c^2}+m\frac{\tau_m}{c}e^{2m\lambda M\max\{1,\lambda M\}}.
\end{align*}
\end{lemma}
\begin{proof}
Note that
\begin{align*}
\bE\Big[e^{2\lambda(\sum_{k=1}^{m-1}D^{(2)}_k)}e^{\frac{1}{2}\lambda^2M^2P_m(z_{1:m-1})}\Big]&\leq\bE\Big[e^{2\lambda(\sum_{k=1}^{m-1}D^{(2)}_k)}e^{\frac{1}{2}\lambda^2M^2}I_{\{P_m(z_{1:m-1})\geq c\}}\Big]\\
&+\bE\Big[e^{2\lambda(\sum_{k=1}^{m-1}D^{(2)}_k)}e^{\frac{1}{2}\lambda^2M^2c^2}I_{\{P_m(z_{1:m-1})< c\}}\Big].
\end{align*}
We do a decomposition as the following
$$1=I_{\{P_k(z_{1:k-1})< c\}}+I_{\{P_k(z_{1:k-1})\geq c\}}$$
for the second term $\bE\Big[e^{2\lambda(\sum_{k=1}^{m-1}D^{(2)}_k)}e^{\frac{1}{2}\lambda^2M^2c^2}I_{\{P_m(z_{1:m-1})< c\}}\Big]$ sequentially until $I_{\{P_{k+1}(z_{1:k})\geq c\}}$ appears for some  $k=1,\ldots,m$, \textit{i.e.},
\begin{align*}
I_{\{P_{k+1}(z_{1:k})< c\}}&=I_{\{P_{k+1}(z_{1:k})< c\}}(I_{\{P_{k}(z_{1:k-1})\geq c\}}+I_{\{P_{k}(z_{1:k-1})< c\}})\\
&=I_{\{P_{k+1}(z_{1:k})< c\}}I_{\{P_{k}(z_{1:k-1})\geq c\}}+I_{\{P_{k+1}(z_{1:k})< c\}}I_{\{P_{k}(z_{1:k-1})< c\}}(I_{\{P_{k-1}(z_{1:k-2})\geq c\}}\\
&+I_{\{P_{k-1}(z_{1:k-2})< c\}})\\        
&=\cdots.
\end{align*}
Besides the term,
$$\bE\Big[e^{2\lambda(\sum_{k=1}^{m}D^{(2)}_k)}\Pi_{j=1}^m I_{\{P_j(z_{1:j-1})< c\}}\Big],$$
which is bounded by $$e^{\frac{1}{2}m\lambda^2M^2c^2},$$
doing such a decomposition will provide extra $m$ terms and the sum of them can be bounded by 
$$m\frac{\tau_m}{c}e^{2m\lambda M\max\{1,\lambda M\}}$$
by applying Lemma \ref{lm:2}.

Taking the sum of them would yield the results.
\end{proof}
Last, we provide a bound for $\tau_m$.
\begin{lemma}\label{lm:4}
Recall
$$E_{\backslash k}=\Big\{S\Big|~\sup_{z'\in \cZ}\Big|\sum_{j\neq k}\frac{\beta(z',z_j)}{m}-\bE_{z}\beta(z',z)\Big|\leq\eta\Big\}.$$
for $\eta>2M_\beta/m$, we have
$$\sup_k\bP(E^c_{\backslash k})\leq \tau_m,$$
where $\tau_m=C\exp(-\frac{m\eta^2}{32M^2_\beta})$ for a constant $C>0$.
\end{lemma}
\begin{proof}
Notice that if $\eta>2M_\beta/m$,
$$\cup_kE^c_{\backslash k}\subseteq \Big\{S\Big|~\sup_{z'\in \cZ}\Big|\sum_{j=1}^m\frac{\beta(z,z_j)}{m}-\bE_{z}\beta(z',z)\Big|\geq \eta/2\Big\}.$$
Thus,
$$\sup_k\bP(E^c_{\backslash k})\leq \bP\left( \Big\{S\Big|~\sup_{z'\in \cZ}\Big|\sum_{j=1}^m\frac{\beta(z',z_j)}{m}-\bE_{z}\beta(z',z)\Big|\geq \eta/2\Big\}\right).$$
Recall the $L$-Lipschitz property on the first argument of $\beta(\cdot,\cdot)$, by the standard epsilon net-argument \citep{wainwright2019high}, and choose $\varepsilon= \eta/(6L)$, we can first obtain via uniform bound such that on a finite $\eta/(6L)$-net of $\cZ$, which we define as $\cZ_\cN$,  with high probability, 
$$\sup_{z'\in \cZ_\cN}\Big|\sum_{j=1}^m\frac{\beta(z',z_j)}{m}-\bE_{z}\beta(z',z)\Big|\geq \eta/6.$$

Then, by Lipschitz condition, for any $z_a$, $z_b$ in each cell, \emph{i.e.} $|z_a-z_b|\le \eta/(6L)$, we have
$$\Big|\sum_{j=1}^m\frac{\beta(z_a,z_j)}{m}-\bE_{z}\beta(z_a,z)\Big|-\Big|\sum_{j=1}^m\frac{\beta(z_b,z_j)}{m}-\bE_{z}\beta(z_b,z)\Big|\le \eta/3.$$

By combining the above two steps, we have the uniform bound on $\cZ$. 

Specifically, we have
$$\bP\left( \Big\{S\Big|~\sup_{z'\in \cZ}\Big|\sum_{j=1}^m\frac{\beta(z',z_j)}{m}-\bE_{z}\beta(z',z)\Big|\geq \eta/2\Big\}\right)=C\exp(-\frac{m\eta^2}{32M^2_\beta}),$$

where 
$$C=\exp(\tilde{C}d\log(Ld/\eta))$$
for a universal constant $\tilde{C}$ and $d$ is the dimension of $z$, $L$ is the Lipschitz constant of the first variable in $\beta(\cdot,\cdot)$ function in assumption.
\end{proof}
\begin{corollary}\label{col:1}
For $\lambda>0$, $c<1$, we have 
\begin{align*}
\bE\big[e^{2\lambda(\sum_{k=1}^{m-1}D^{(2)}_k)}\big]\leq e^{\frac{1}{2}m\lambda^2M^2c^2}+\frac{m}{c}C\exp(-\frac{m\eta^2}{32M^2_\beta})e^{2m\lambda M\max\{1,\lambda M\}}
\end{align*}
for a constant $C$ that depends on $L$ and $d$.
\begin{proof}
Combining the results of Lemma \ref{lm:3} and \ref{lm:4}, we obtain the bound.
\end{proof}
\end{corollary}

\paragraph{Bounding $\bE[e^{2\lambda(\sum_{k=1}^{m}D^{(1)}_k)}]$.}
Now, we consider bounding
$$\bE[e^{2\lambda(\sum_{k=1}^{m}D^{(1)}_k)}].$$
We further define 
$$A^{(1)}_k=\inf_x \bE[L(S)I_{E_{\backslash k}}|z_1,\cdots,z_{k-1},z_k=x]-\bE[L(S)I_{E_{\backslash k}}|z_1,\cdots,z_{k-1}],$$
$$B^{(1)}_k=\sup_x \bE[L(S)I_{E_{\backslash k}}|z_1,\cdots,z_{k-1},z_k=x]-\bE[L(S)I_{E_{\backslash k}}|z_1,\cdots,z_{k-1}].$$
Again we have, 
$$A^{(1)}_k\leq D^{(1)}_k\leq B^{(1)}_k. $$
Similarly
\begin{align*}
B^{(1)}_k-A^{(1)}_k
&\leq\sup_{x,y} \bE[L(S)I_{E_{\backslash k}}|z_1,\cdots,z_{k-1},z_k=x]-\bE[L(S^{\backslash k})I_{E_{\backslash k}}|z_1,\cdots,z_{k-1},z_k=x]\\
&+\bE[L(S^{\backslash k})I_{E_{\backslash k}}|z_1,\cdots,z_{k-1},z_k=x]-\bE[L(S^{\backslash k})I_{E_{\backslash k}}|z_1,\cdots,z_{k-1},z_k=y]\\
&+\bE[L(S^{\backslash k})I_{E_{\backslash k}}|z_1,\cdots,z_{k-1},z_k=y]- \bE[L(S)I_{E_{\backslash k}}|z_1,\cdots,z_{k-1},z_k=y].
\end{align*}

By the nature of $E_{\backslash k}$, and the boundedness conditions that $|\beta(\cdot,\cdot)|\leq M_\beta$, $0\leq l(\cdot,\cdot)\leq M_l$
\begin{align*}
&\bE[L(S)I_{E_{\backslash k}}-L(S^{\backslash k})I_{E_{\backslash k}}|z_1,\cdots,z_{k-1},z_k=x]+\bE[L(S^{\backslash k})I_{E_{\backslash k}}-L(S)I_{E_{\backslash k}}|z_1,\cdots,z_{k-1},z_k=y]\\
&\leq 2\sup_{z'\in \cZ}\bE_{z}\beta(z',z)+2\eta+M_l.
\end{align*}

Besides, 
$$\bE[L(S^{\backslash k})I_{E_{\backslash k}}|z_1,\cdots,z_{k-1},z_k=x]-\bE[L(S^{\backslash k})I_{E_{\backslash k}}|z_1,\cdots,z_{k-1},z_k=y]=0.$$
As a result,
$$B^{(1)}_k-A^{(1)}_k\leq 2 \sup_{z'\in \cZ}\bE_{z}\beta(z',z)+2\eta+M_l.$$
Then, by the standard argument of concentration of martingale differences, we have the following lemma.
\begin{lemma}\label{lm:5}
For any $\lambda\in\bR$, if we denote $\tilde{M}=2 \sup_{z'\in \cZ}\bE_{z}\beta(z',z)+2\eta+M_l$
$$\bE[e^{2\lambda(\sum_{k=1}^{m}D^{(1)}_k)}]\leq e^{\frac{1}{2}m\lambda^2\tilde{M}^2} .$$
\end{lemma}
Combined the previous results, we provide the following lemma.
\begin{lemma}
For $\lambda>0, c>0$, there exists a constant $C>0$ depending on $\eta$, $d$ and $L$
$$\bE[e^{\lambda(\sum_{k=1}^{m}D_k)}]\leq \frac{1}{2}e^{\frac{1}{2}m\lambda^2\tilde{M}^2}+\frac{1}{2}\big( e^{\frac{1}{2}m\lambda^2M^2c^2}+Cm\frac{\tau_m}{c}e^{2m\lambda M\max\{1,\lambda M\}}\big) .$$
\end{lemma}
\begin{proof}
Combining Lemma \ref{lm:5} and Corollary \ref{col:1}, and the fact that 
\begin{align*}
\bE[e^{\lambda(\sum_{k=1}^mD_k)}]&\leq \frac{1}{2}\bE[e^{2\lambda(\sum_{k=1}^{m}D^{(1)}_k)}]+\frac{1}{2}\bE[e^{2\lambda(\sum_{k=1}^{m}D^{(2)}_k)}],
\end{align*}
we obtain the above bound.
\end{proof}

\paragraph{A Concentration Bound.}
If we choose $c=\tilde{M}/M'$, where $\tilde{M}=2 \sup_{z\in \cZ}\bE_{z_j}\beta(z,z_j)+2\eta+M_l$ and $M'=2M_\beta+2\eta+M_l$. It is easy to see $c<1$. We denote $\gamma_m=Cm\frac{\tau_m}{2c}e^{2m\lambda M\max\{1,\lambda M\}}$.
By Chernoff-bound, for any $\lambda>0$
$$\bP\Big(L(S)-\bE[L(S)]\geq m\varepsilon\Big)\leq \frac{e^{\frac{1}{2}m\lambda^2\tilde{M}^2}+\gamma_m}{e^{\lambda m\varepsilon}}.$$
Let us take $\lambda=\varepsilon/\tilde{M}^2$ for $\varepsilon>0$, when $m$ is large enough and $\varepsilon$ is small enough, we expect $\gamma_m\leq  e^{\frac{1}{2}m\lambda^2\tilde{M}^2}$. Specifically,
$$\frac{Cm\tau_m}{2c}\leq \exp(\frac{1}{2}m\lambda(\tilde{M}^2\lambda-4M\max\{1,\lambda M\})).$$
Recall $\tau_m=C\exp(-\frac{m\eta^2}{32M^2_\beta}),$ plugging in $\lambda=\varepsilon/\tilde{M}^2$ and $\tau_m$, let $C'=C/(2c)$ it is sufficient to let $m$ large enough and $\varepsilon$ small enough such that
$$C'm\exp(-\frac{m\eta^2}{32M^2_\beta})\leq \exp\left(\frac{m\varepsilon}{2\tilde{M}^2}(\varepsilon-\frac{4\varepsilon M^2}{\tilde{M}^2}-4M)\right).$$
which can be further simplified as 
$$\frac{\eta^2}{32M^2_\beta}-\frac{\log C' m}{m}\geq \frac{\varepsilon}{2\tilde{M}^2}(-\varepsilon+\frac{4\varepsilon M^2}{\tilde{M}^2}+4M).$$
 That will lead to
\begin{align*}
\bP\Big(L(S)-\bE[L(S)]\geq m\varepsilon\Big)&\leq \frac{e^{\frac{1}{2}m\lambda^2\tilde{M}^2}+\gamma_m}{e^{\lambda m\varepsilon}}\\
&\leq 2\exp(-\frac{m\varepsilon^2}{2\tilde{M}^2}).
\end{align*}

\textbf{[Proof of Lemma \ref{lm:alt}]}
\begin{proof}
Now let us consider $\tilde{l}(\cA_S, z)=-l(\cA_S, z)+\bE_z[l(\cA_S, z)]$. We remark here as long as $0\leq l\leq M_l$ (not $\tilde{l}$).

In addition, if we denote
 $$\tilde{L}(S)=\sum_{j=1}^m\tilde{l}(\cA_{S},z_j),~\tilde{L}(S^{\backslash i})=\sum_{j\neq i}^m\tilde{l}(\cA_{S^{\backslash i}},z_j),$$ 
 we have
\begin{align*}
&\bE[\tilde{L}(S)I_{E^c_{\backslash k}}-\tilde{L}(S^{\backslash k})I_{E^c_{\backslash k}}|z_1,\cdots,z_{k-1},z_k=x]+\bE[\tilde{L}(S^{\backslash k})I_{E^c_{\backslash k}}-\tilde{L}(S)I_{E^c_{\backslash k}}|z_1,\cdots,z_{k-1},z_k=y]\\
&\leq( 2M_\beta+2\sup_{z\in \cZ}\bE_{z'}\beta(z,z')+2M_l)\bP(E^c_{\backslash k}|z_1,\cdots,z_{k-1}).
\end{align*}
and
 \begin{align*}
&\bE[\tilde{L}(S)I_{E_{\backslash k}}-\tilde{L}(S^{\backslash k})I_{E_{\backslash k}}|z_1,\cdots,z_{k-1},z_k=x]+\bE[\tilde{L}(S^{\backslash k})I_{E_{\backslash k}}-\tilde{L}(S)I_{E_{\backslash k}}|z_1,\cdots,z_{k-1},z_k=y]\\
&\leq( 4\sup_{z\in \cZ}\bE_{z_j}\beta(z,z_j)+2\eta +2M_l)\bP(E_{\backslash k}|z_1,\cdots,z_{k-1}).
\end{align*}

Thus, the generalization argument is exactly the same as theory above except the value of $M$,  and $\tilde{M}$.

Specifically, we choose $M=2( M_\beta+\sup_{z'\in \cZ}\bE_{z}\beta(z',z)+M_l)$ and $\tilde{M}= 2(2 \sup_{z'\in \cZ}\bE_{z}\beta(z',z)+\eta+M_l)$.

Next, by Lemma $7$ in \citet{bousquet2002stability}, for an independent copy of $z_j$, which we denote it as $z'_j$, we have
\begin{align*}
\bE_S\big[-\frac{1}{m}\sum_{j=1}^ml(\cA_S, z_j)+\bE_z[l(\cA_S, z)]\big]&\leq\bE_{S,z'_j}\big[|l(\cA_S, z'_j)-l(\cA_{S^j}, z'_j)|\big]\\
&\leq\frac{2\sup_{z'\in \cZ}\bE_{z}\beta(z',z)}{m}.
\end{align*}
Then, the result follows.
\end{proof}

\subsection{Proof of Lemma \ref{lm:kh}}
For simplicity, let us introduce the following two notations:
\begin{align*}
R_r(g):=\frac{1}{m}\sum_{j=1}^ml(g,z_j)+\lambda \|g\|^2_K,\\
R^{\backslash i}_r(g):= \frac{1}{m}\sum_{j\neq i}^ml(g,z_j)+\lambda \|g\|^2_K.
\end{align*}
Let us denote $f$ as a minimizer of $R_r$ in $\cF$ and $f^{\backslash i}$ as a minimizer of $R^{\backslash i}_r$. We further denote $\Delta f = f^{\backslash i}-f$.

By Lemma 20 in \citet{bousquet2002stability}, we have 
$$2\|\Delta f\|^2_K\le \frac{\sigma}{\lambda m}|\Delta f(x_i)|.$$

Furthermore since
$$|f(x_i)|\le \|f\|_K\sqrt{K(x_i,x_i)}\le \|f\|_K\kappa(x_i),$$
we have 
$$\|\Delta f\|_K\le \frac{\kappa(x_i)\sigma}{2\lambda m}.$$

By the $\sigma$-admissibility of $l$,
$$|l(f,z)-l(f^{\backslash i},z)|\le \sigma |f(x)-f^{\backslash i}(x)|=\sigma |\Delta f(x)|\le \sigma \|\Delta f\|_K \kappa(x)\leq \frac{\sigma^2 \kappa(x)\kappa(x_i)}{2\lambda m}.$$

\subsection{Proof of Locally Elastic Stability of SGD}
\label{subsec:proof-SGD}
In this section, we establish our new notion of algorithm stability -- {\em locally elastic stability} for SGD. Specifically, we consider the quantity
$$|\bE_{\cA}[l(\cA_S,z)]-\bE_{\cA}[l(\cA_{S^{\backslash i}},z)]|.$$

Here, the expectation is taken over the internal randomness of $\cA$. The randomness comes from the selection of sample at each step of SGD. Specifically, $\cA_S$ returns a parameter $\theta_T$, where $T$ is the number of iterations. And for dataset $S$ with sample size $m$, SGD is performed in the following way:
$$\theta_{t+1}=\theta_t-\eta_t\nabla_{\theta}l(\theta_t,z_{i_t}),$$
where $\eta_t$ is the learning rate at time $t$, $i_t$ is picked uniformly at random in $\{1,\cdots,m\}$.

We denote  
$$L(z)=\sup_{\theta\in \Theta}\|\nabla_\theta l(\theta,z)\|$$
We further denote the gradient update rule 
$$G_{l,\eta}(\theta,z)=\theta-\eta\nabla_\theta l(\theta,z).$$
We consider gradient updates $G_1,\cdots,G_T$ and $G'_1,\cdots,G'_T$ induced by running SGD on $S$ and $S^{\backslash i}$. Most of the proofs are similar to \citet{hardt2015train}, the only difference is that for dataset $S$ and $S^{\backslash i}$, the randomness in $\cA$ is different. For $S$, $i_t$ is randomly picked in $\{1,\cdots,m\}$ but for $S^{\backslash i}$, $i_t$ is randomly picked in $\{1,\cdots,i-1,i+1,\cdots,m\}$. Thus, we create two coupling sequences for the updates on $S$ and $S^{\backslash i}$. Notice choosing any coupling sequences will not affect the value of  $|\bE_{\cA}[l(\cA_S,z)]-\bE_{\cA}[l(\cA_{S^{\backslash i}},z)]|,$ since the expectations are taken with respect to $l(\cA_S,z)$ and $l(\cA_{S^{\backslash i}},z)$ separately.

\paragraph{Convex optimization.}
We first show SGD satisfies locally elastic stability for convex loss minimization.
\begin{proposition}[Restatement of Proposition 2]
Assume that the loss function $l(\cdot, z)$ is $\alpha$-smooth and convex for all $z\in\cZ$. In addition, $l(\cdot, z)$ is $L(z)$-Lipschitz and $L(z)<\infty$ for all $z\in\cZ$: $|l(\theta,z)-l(\theta',z)|\leq L(z)\|\theta-\theta'\|$ for all $\theta, \theta'.$ We further assume $L=\sup_{z\in\cZ}L(z)<\infty$. Suppose that we run SGD with step sizes $\eta_t\leq 2/\alpha$ for $T$ steps. Then, 
\begin{equation*}
|\bE[l(\hat{\theta}_T,z)]-\bE[l(\hat{\theta}^{\backslash i}_T,z)]|\leq\frac{(L+L(z_i))L(z)}{m}\sum_{t=1}^T\eta_t.
\end{equation*}
 \end{proposition}
\begin{proof}
Notice for $i_t$ which is randomly picked in $\{1,\cdots,i-1,i+1,\cdots,m\}$, we can view it as a two-phase process. Firstly, draw $i_t$ uniformly from a $n$-element set $\{1,\cdots,i-1,i',i+1,\cdots,m\}$. If any element but $i'$ is drawn, directly output it. Otherwise if $i'$ is drawn, then uniformly draw again from  $\{1,\cdots,i-1,i+1,\cdots,m\}$, and output the final index that is drawn. It is not hard to notice that in this way, each index in $\{1,\cdots,i-1,i+1,\cdots,m\}$ has probability 
$$\frac{1}{m}+\frac{1}{m(m-1)}=\frac{1}{m-1}$$
to be drawn, which is the same as directly uniformly draw from $\{1,\cdots,i-1,i+1,\cdots,m\}$.

We consider two coupling processes of SGD on $S$ and $S^{\backslash i}$. The randomness of uniformly drawing from $n$ elements for SGD on $S$ and uniformly drawing from  $\{1,\cdots,i-1,i',i+1,\cdots,n\}$ for SGD on $S^{\backslash i}$ share the same random seed $\xi_t$ at each iteration at time $t$. That will not affect the value of 
$$|\bE_{\cA}[l(A_S,z)]-\bE_{\cA}[l(A_{S^{\backslash i}},z)]|.$$
Let $\delta_t=\|\hat{\theta}_t-\hat{\theta'}_t\|$, where $\hat{\theta}_t$ is the parameter obtained by SGD on $S$ at iteration $t$ and $\hat{\theta'}_t$ is the parameter by SGD on $S^{\backslash i}$ obtained at iteration $t$.

With probability $1/n$ the selected example is different, in that case we use the fact that
$$\delta_{t+1}\leq \delta_t+\eta_t L(z_i)+ \eta_t L(z_j)$$
for some $j\neq i$, which can further be upper bounded by $ \eta_t(L(z_i)+L)$. With probability $1-1/n$, the selected example is the same, then we can apply Lemma 3.7 in \citet{hardt2015train} regarding 1-expansivity of the update rule $G_t$, then we have 
$$\bE[\delta_{t+1}]\leq  \left(1-\frac{1}{m}\right)\bE[\delta_{t}]+\frac{1}{m}\bE[\delta_{t}]+\frac{\eta_t (L(z_i)+L)}{m}.$$
This technique is used repeatedly in the following other theorems, we will not further elaborate it.

Then, unraveling the recursion, by the fact that $L(z)$ continuity of $l(\theta,z)$ for any $\theta$, we obtain 
\begin{equation*}
|\bE[l(\hat{\theta}_T,z)]-\bE[l(\hat{\theta}^{\backslash i}_T,z)]|\leq\frac{(L+L(z_i))L(z)}{m}\sum_{t=1}^T\eta_t.
\end{equation*}
\end{proof}

\paragraph{Strongly convex optimization.}
We consider the penalized loss discussed in \citet{hardt2015train}:
$$\frac{1}{n}\sum_{i=1}^nl(\theta,z_i)+\frac{\mu}{2}\|\theta\|^2_2,$$
where $l(\theta,z)$ is convex with respect to $\theta$ for all $z$. And without loss of generality, we assume $\Theta$ is a ball with radius $r$ (this can be obtained by the boundedness of loss $l$) and apply stochastic projected gradient descent:
$$\theta_{t+1}=\Pi_{\Theta}(\theta_t-\eta_t\nabla \tilde{l}(\theta_t,z_{i_t}) )$$
where $\tilde{l}(\theta,z)=l(\theta,z)+\frac{\mu}{2}\|\theta\|^2_2$.
\begin{proposition}[Strongly Convex Optimization]
Assume that the loss function $l(\cdot, z)$ is $\alpha$-smooth and $\mu$-strongly convex for all $z\in\cZ$. In addition, $l(\cdot, z)$ is $L(z)$-Lipschitz and $L(z)<\infty$ for all $z\in\cZ$: $|l(\theta,z)-l(\theta',z)|\leq L(z)\|\theta-\theta'\|$ for all $\theta, \theta'.$ We further assume $L=\sup_{z\in\cZ}L(z)<\infty$. Suppose that we run SGD with step sizes $\eta\leq 1/\alpha$ for $T$ steps. Then, 
$$\bE|\tilde{l}(\hat{\theta}_T;z)-\tilde{l}(\hat{\theta'}_T;z)|\leq \frac{(L(z_i)+L)L(z)}{m\mu}$$
 \end{proposition}

\begin{proof}
By using the same coupling method, we have when the learning rate $\eta\mu\leq 1$, By further similarly applying Lemma 3.7 and method in the Theorem 3.9 in \citet{hardt2015train}, we have
$$\bE[\delta_{t+1}]\leq \left(1-\frac{1}{m}\right)(1-\eta\mu)\bE[\delta_{t}]+\frac{1}{m}(1-\eta\mu)\bE[\delta_{t}]+\frac{\eta (L(z_i)+L)}{m}.$$
Unraveling the recursion gives,
$$\bE[\delta_T]\leq \frac{L(z_i)+L}{m}\sum_{t=0}^T(1-\eta\mu)^t\leq\frac{L(z_i)+L}{m\mu} .$$
Plugging the above inequality, we obtain 
$$\bE|\tilde{l}(\hat{\theta}_T;z)-\tilde{l}(\hat{\theta'}_T;z)|\leq \frac{(L(z_i)+L)L(z)}{m\mu}$$

\end{proof}

\paragraph{Non-convex optimization} Lastly, we show the case of non-convex optimization.
\begin{proposition}[Restatement of Proposition 3]
Assume that the loss function $l(\cdot, z)$ is non-negative and bounded for all $z\in\cZ$. Without loss of generality, we assume $l(\cdot, z)\in[0,1]$. In addition, we assume $l(\cdot, z)$ is $\alpha$-smooth and convex for all $z\in\cZ$. We further assume $l(\cdot, z)$ is $L(z)$-Lipschitz and $L(z)<\infty$ for all $z\in\cZ$ and $L=\sup_{z\in\cZ}L(z)<\infty$. Suppose that we run SGD for $T$ steps with monotonically non-increasing step sizes $\eta_t\leq c/t$ for some constant $c>0$. Then, 
\begin{equation*}
|\bE[l(\hat{\theta}_T,z)]-\bE[l(\hat{\theta}^{\backslash i}_T,z)]|\leq\frac{1+1/(\alpha c)}{m-1}[c(L(z_i)+L)L(z)]^{\frac{1}{\alpha c +1}}T^{\frac{\alpha c}{\alpha c+1}}.
\end{equation*}
\end{proposition}

\begin{proof} 
By Lemma 3.11 in \citet{hardt2015train}, for every $t_0\in\{1,\cdots,T\}$ and switch $L$ to $L(z)$, we have
$$|\bE[l(\hat{\theta}_T,z)]-\bE[l(\hat{\theta}^{\backslash i}_T,z)]|\leq \frac{t_0}{m}+L(z)\bE[\delta_T|\delta_{t_0}=0].$$
Let $\Delta_t=\bE[\delta_t|\delta_{t_0}=0]$. By applying Lemma 3.7 and method in the Theorem 3.12  in \citet{hardt2015train}, combining the fact regarding boundedness of the gradient -- $\|\nabla_\theta l(\hat{\theta}_t,z_i)\|\leq L(z_i) $, $\sup_{j\neq i}\|\nabla_\theta l(\hat{\theta}_t,z_j)\|\leq L$, we have 

\begin{align*}
\Delta_{t+1}&\leq\left(1-\frac{1}{m}\right)(1+\eta_t\alpha)\Delta_t+\frac{1}{m}\Delta_t+\frac{\eta_t(L+L(z_i))}{m}\\
&\leq \left(\frac{1}{m}+(1-1/m)(1+c\alpha/t)\right)\Delta_t+\frac{c(L+L(z_i))}{tm}\\
&=\left(1+(1-1/m)\frac{c\alpha}{t}\right)\Delta_t+\frac{c(L+L(z_i))}{tm}\\
&\leq \exp((1-1/m)\frac{c\alpha}{t})\Delta_t+\frac{c(L+L(z_i))}{tm}
\end{align*}

By the fact $\Delta_0=0$, we can unwind this recurrence relation from $T$ down to $t_0+1$, it is easy to obtain
$$|\bE[l(\hat{\theta}_T,z)]-\bE[l(\hat{\theta}^{\backslash i}_T,z)]|\leq \frac{1+1/(\alpha c)}{m-1}[c(L(z_i)+L)L(z)]^{\frac{1}{\alpha c +1}}T^{\frac{\alpha c}{\alpha c+1}}.$$
\end{proof}

\section{More about Experiments}
\label{sec:experiments}

\begin{figure}[t]
\vskip -0.1in
\begin{center}
\centerline{\includegraphics[width=4in]{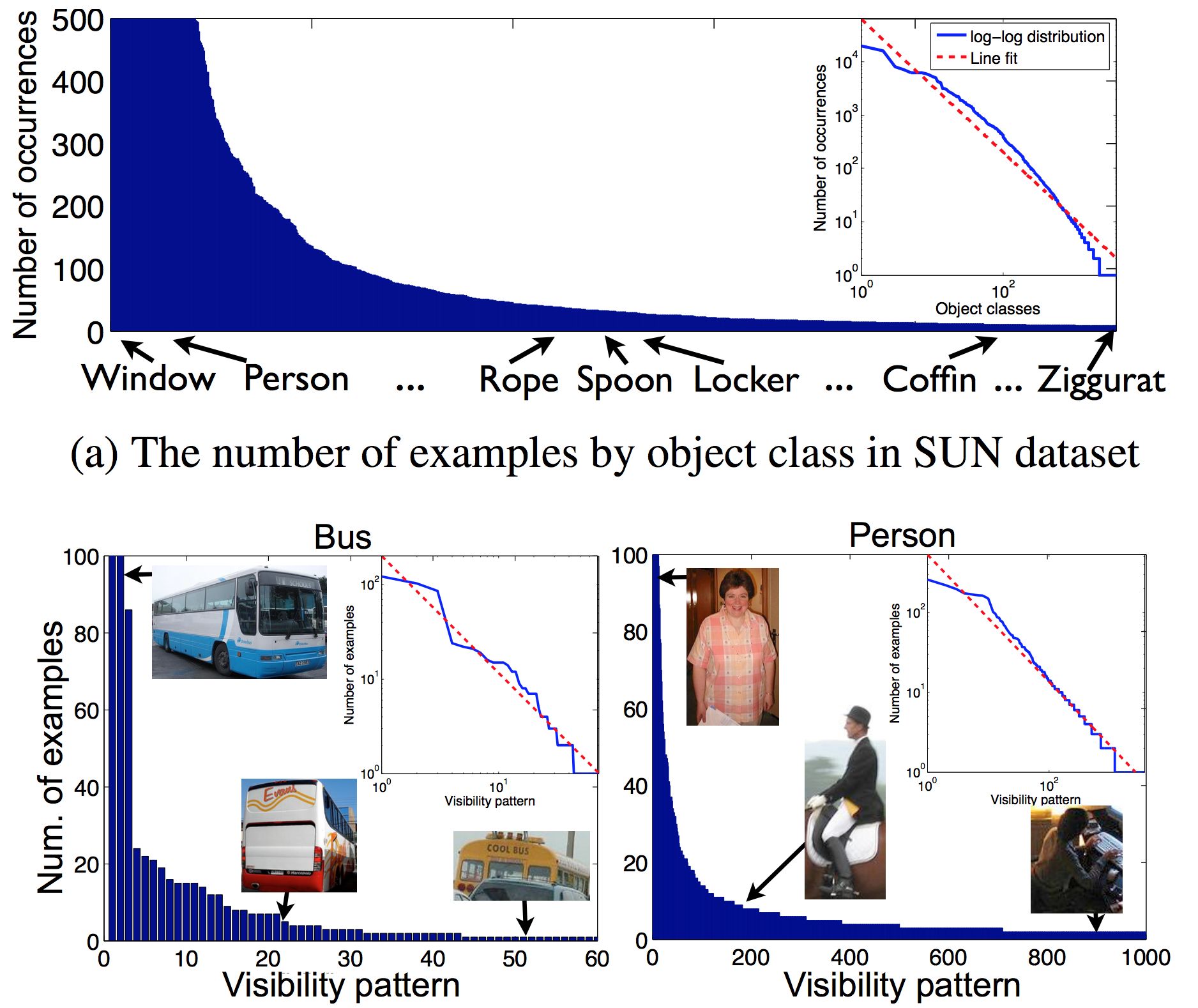}}
\caption{Long tail empirical distribution of classes and subpopulations within classes, taken from \cite{zhu2014capturing} with the authors' permission. }\label{fig:longtail}
\end{center}
\vskip -0.2in
\end{figure}

\begin{figure*}[t]
\centering
		 \subfigure[Sensitivity of neural networks.]{			\centering			\includegraphics[scale=0.33]{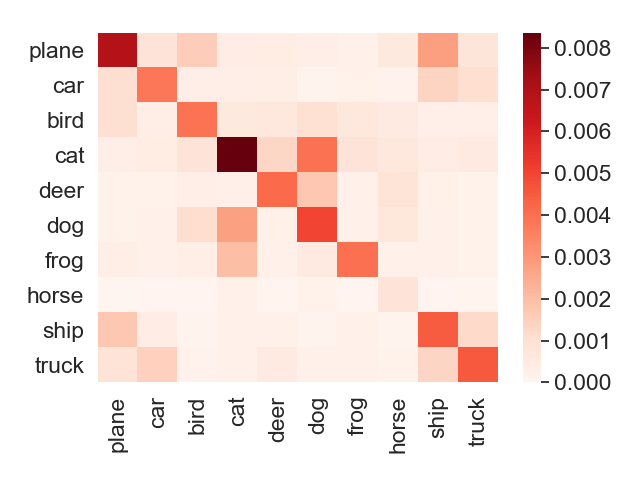}
		\label{fig:NNs-IF-appx}}
		\subfigure[Sensitivity of a random feature model.]{
			\centering
			\includegraphics[scale=0.33]{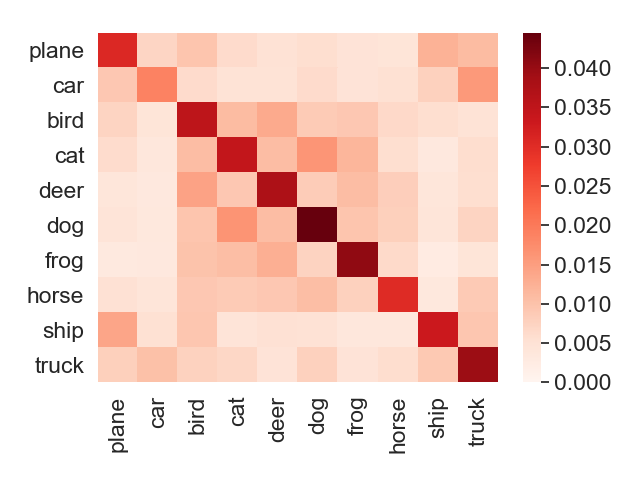}
			\label{fig:RF-IF-appx}}
      \subfigure[Sensitivity of a linear model.]{
			\centering
		\includegraphics[scale=0.33]{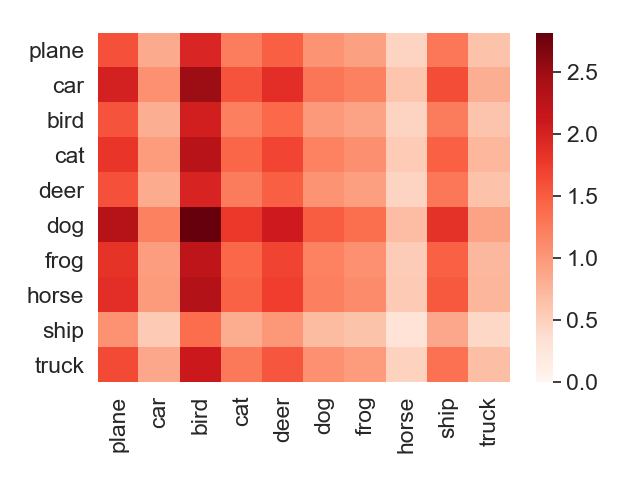}
		\label{fig:Linear-IF-appx}}
		\caption{Class-level sensitivity approximated by influence functions for 
		neural networks (based on a pre-trained $18$-layer ResNet), a random feature model (based on a randomly initialized $18$-layer ResNet),
		and a linear model on CIFAR-10. The vertical axis denotes the classes in the test data and the horizontal  axis denotes the classes in the training data. The class-level sensitivity from class $a$ in the training data to class $b$ in the test data is defined as $C(c_a, c_b) = \frac{1}{|S_a| \times |\tilde{S}_b|}\sum_{z_i \in S_a} \sum_{z \in \tilde{S}_b}|l(\hat{\theta},z)-l(\hat{\theta}^{\backslash i},z)|$, where $S_a$ denotes the set of examples from class $a$ in the training data and $\tilde{S}_b$ denotes set of examples from class $b$ in the test data. 
		}
		\label{fig:IF-results-appx}
\end{figure*}

\begin{figure*}[t]
		\centering
		 \subfigure[Influence for neural networks.]{
			\centering
			\includegraphics[scale=0.33]{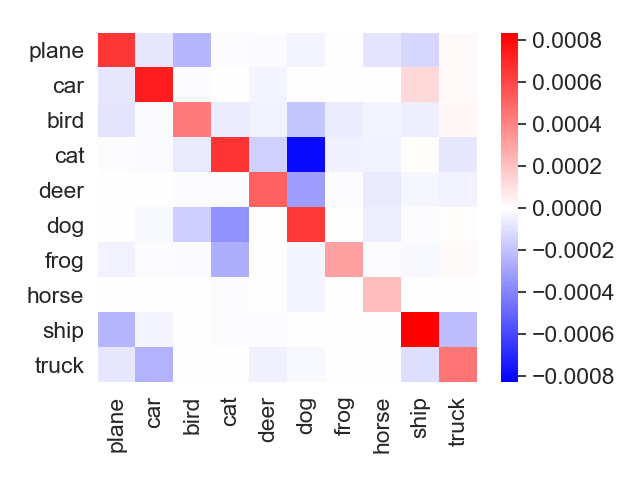}
			\label{fig:NNs-IF-sign-appx}}
		 \subfigure[Influence for the random feature model.]{
			\centering
			\includegraphics[scale=0.33]{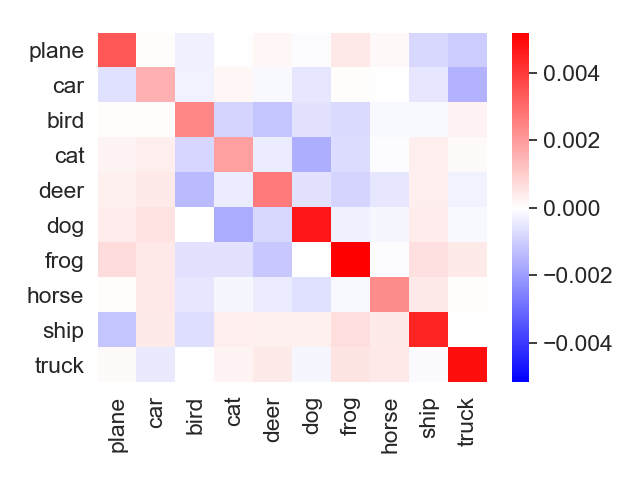}
			\label{fig:RF-IF-sign-appx}}
        \subfigure[Influence for the linear model.]{
			\centering
			\includegraphics[scale=0.33]{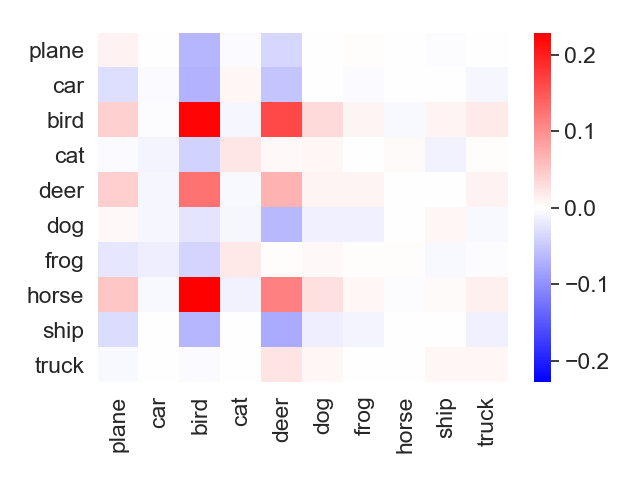}
			\label{fig:Linear-IF-sign-appx}}
		\caption{Class-level sensitivity approximated by influence function for 
		neural networks (based on a pre-trained $18$-layer ResNet), a random feature model (based on a randomly initialized $18$-layer ResNet),
		and a linear model on CIFAR-10. Note that the sensitivity here is based on sign values ($l(\hat{\theta},z)-l(\hat{\theta}^{\backslash i},z)$) instead of absolute values ($|l(\hat{\theta},z)-l(\hat{\theta}^{\backslash i},z)|$) as in Eq. (\ref{eq:IF}). } 
		\label{fig:IF-sign-results-appx}
\end{figure*}

\begin{figure*}[t]
		\centering
		 \subfigure[Neural networks (epoch $0$).]{
			\centering
			\includegraphics[scale=0.3]{Figs/test_train_absolute_influence_100_100_0.png}
			\label{fig:NNs-SGD-0-appx}}
        \subfigure[Neural networks (epoch $10$).]{
			\centering
			\includegraphics[scale=0.3]{Figs/test_train_absolute_influence_100_100_10.png}
			\label{fig:NNs-SGD-10-appx}}
		\subfigure[Neural networks (epoch $50$).]{
			\centering
			\includegraphics[scale=0.3]{Figs/test_train_absolute_influence_100_100_50.png}
			\label{fig:NNs-SGD-50-appx}} \\
		\subfigure[Random feature model (epoch $0$).]{
			\centering
			\includegraphics[scale=0.3]{Figs/test_train_absolute_influence_random_feature_100_100_0.png}
			\label{fig:RF-SGD-0-appx}}
        \subfigure[Random feature model (epoch $50$).]{
			\centering
			\includegraphics[scale=0.3]{Figs/test_train_absolute_influence_random_feature_100_100_50.png}
			\label{fig:RF-SGD-50-appx}}
		\subfigure[Random feature model (epoch $250$).]{
			\centering
			\includegraphics[scale=0.3]{Figs/test_train_absolute_influence_random_feature_100_100_250.png}
			\label{fig:RF-SGD-250-appx}} \\
		\subfigure[Linear model (epoch $0$).]{
			\centering
			\includegraphics[scale=0.3]{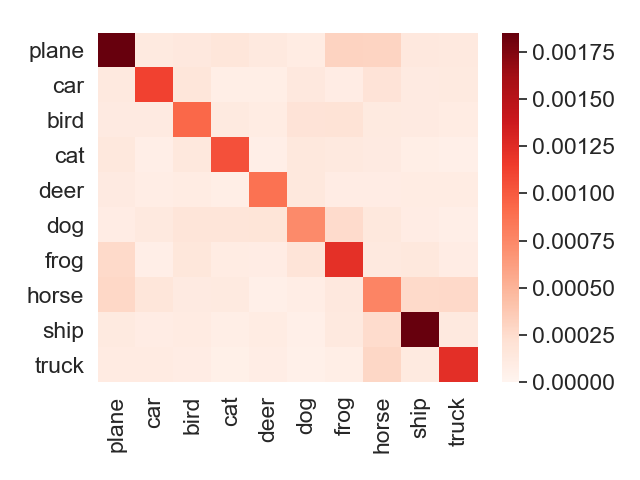}
			\label{fig:Linear-SGD-0-appx}}
        \subfigure[Linear model (epoch $10$).]{
			\centering
			\includegraphics[scale=0.3]{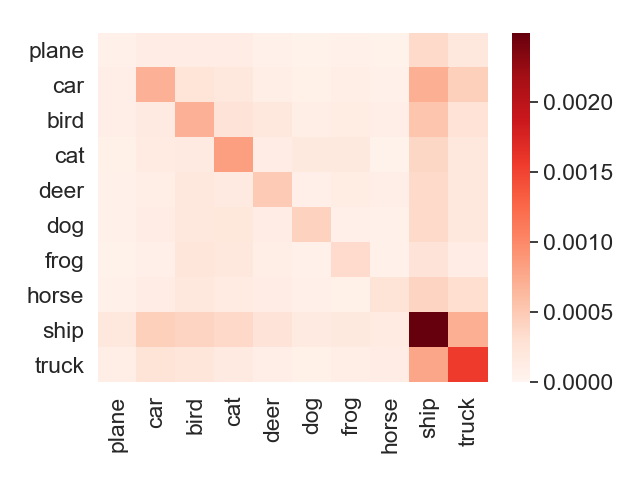}
			\label{fig:Linear-SGD-10-appx}}
		\subfigure[Linear model (epoch $50$).]{
			\centering
			\includegraphics[scale=0.3]{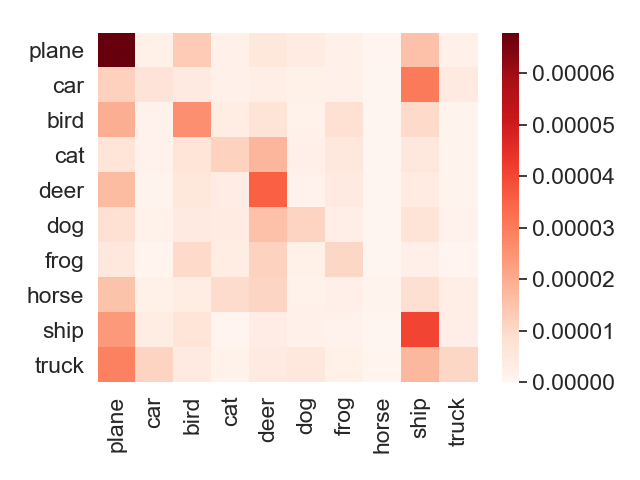}
			\label{fig:Linear-SGD-50-appx}}
		\caption{Exact stepwise characterization of class-level sensitivity for neural networks, random feature models, and linear models trained with different numbers of epochs by SGD on CIFAR-10. 
		The class-level sensitivity for a stepwise update of SGD is $C'(c_a, c_b) = \frac{1}{|S_a|\cdot |\tilde{S}_b|}\sum_{z_i \in S_a} \sum_{z \in \tilde{S}_b}|l(\htheta_t-\eta\nabla_\theta l(\htheta_t,z_i),z)-l(\htheta_t,z)|$, where $S_a$ denotes the set of examples with class $a$ in the training data and $\tilde{S}_b$ denotes the set of examples with class $b$ in the test data. 
		}
		\label{fig:SGD-results-appx}
\end{figure*}

\begin{figure*}[t]
		\centering
		 \subfigure[Neural networks (epoch $0$).]{
			\centering
			\includegraphics[scale=0.3]{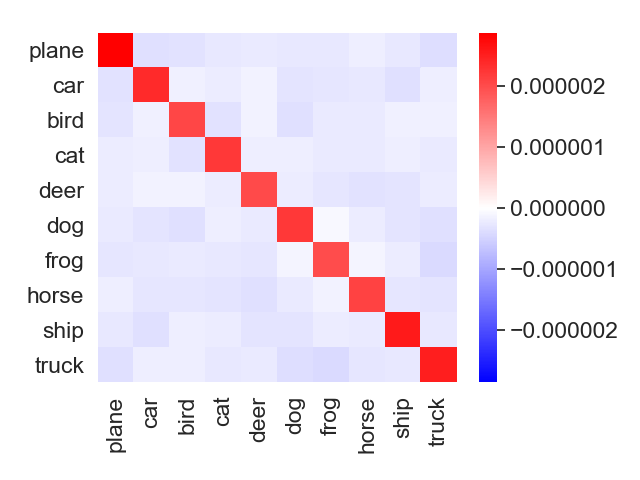}
			\label{fig:NNs-SGD-0-sign-appx}}
        \subfigure[Neural networks (epoch $10$).]{
			\centering
			\includegraphics[scale=0.3]{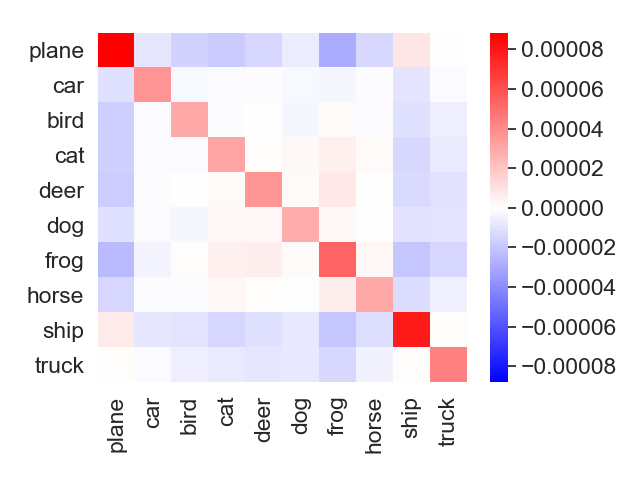}
			\label{fig:NNs-SGD-10-sign-appx}}
		\subfigure[Neural networks (epoch $50$).]{
			\centering
			\includegraphics[scale=0.3]{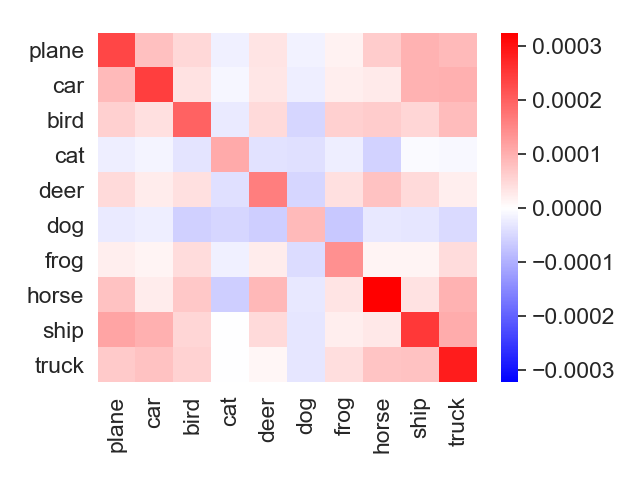}
			\label{fig:NNs-SGD-50-sign-appx}} \\
		\subfigure[Random feature model (epoch $0$).]{
			\centering
			\includegraphics[scale=0.3]{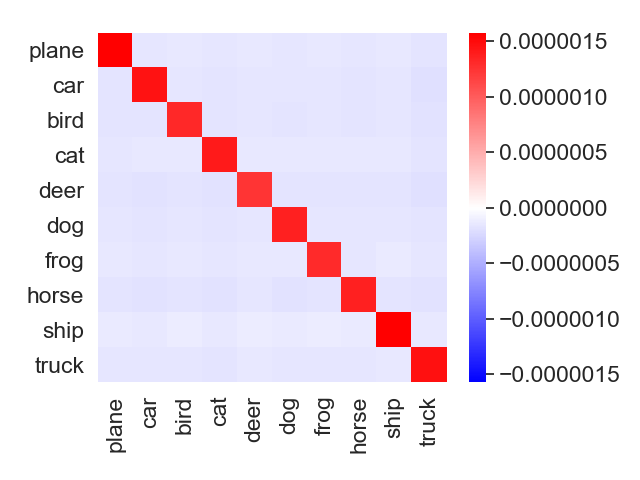}
			\label{fig:RF-SGD-0-sign-appx}}
        \subfigure[Random feature model (epoch $50$).]{
			\centering
			\includegraphics[scale=0.3]{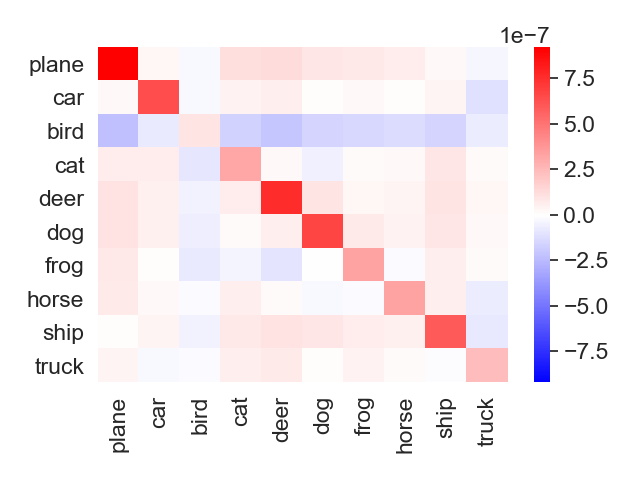}
			\label{fig:RF-SGD-50-sign-appx}}
		\subfigure[Random feature model (epoch $250$).]{
			\centering
			\includegraphics[scale=0.3]{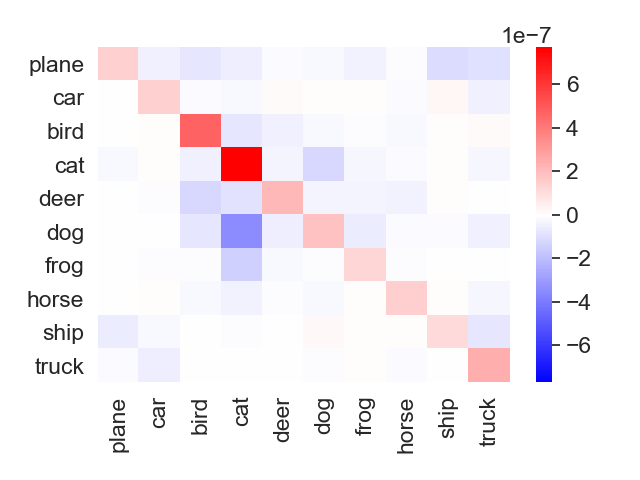}
			\label{fig:RF-SGD-250-sign-appx}} \\
		\subfigure[Linear model (epoch $0$).]{
			\centering
			\includegraphics[scale=0.3]{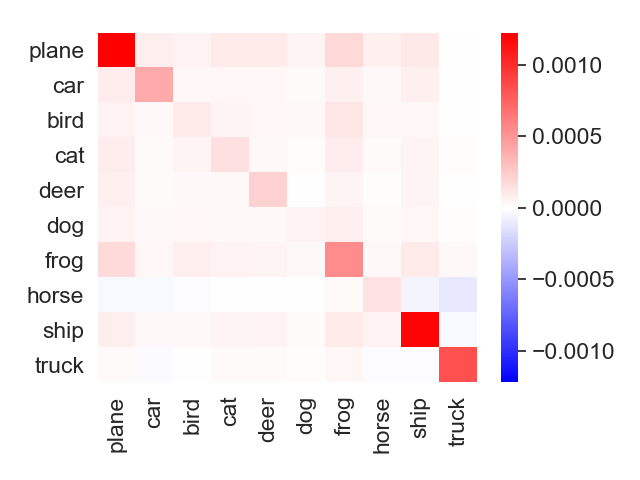}
			\label{fig:Linear-SGD-0-sign-appx}}
        \subfigure[Linear model (epoch $10$).]{
			\centering
			\includegraphics[scale=0.3]{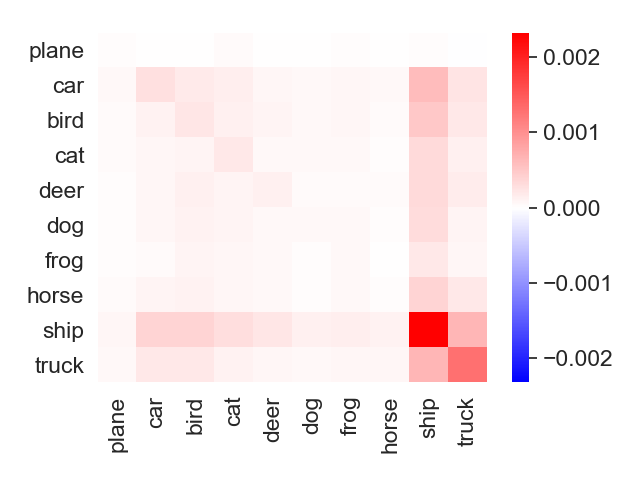}
			\label{fig:Linear-SGD-10-sign-appx}}
		\subfigure[Linear model (epoch $50$).]{
			\centering
			\includegraphics[scale=0.3]{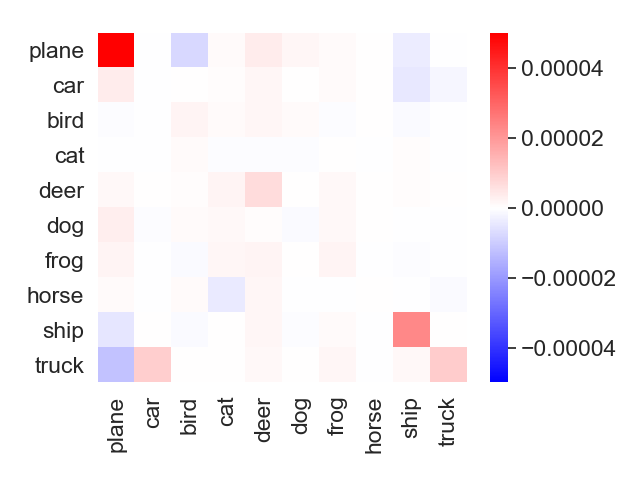}
			\label{fig:Linear-SGD-50-sign-appx}}
		\caption{Exact step-wise characterization of class-level sensitivity for neural networks, random feature models, and linear models trained with different numbers of epochs by SGD on CIFAR-10.  Note that the sensitivity here is based on sign values ($l(\hat{\theta},z)-l(\hat{\theta}^{\backslash i},z)$) instead of absolute values ($|l(\hat{\theta},z)-l(\hat{\theta}^{\backslash i},z)|$) as in Eq. (\ref{eq:IF}).}
		\label{fig:SGD-sign-results-appx}
\end{figure*}

To verify the effectiveness of our proposed locally elastic stability, we conduct experiments on the real-world CIFAR-10 dataset. In our experiments, we randomly choose $100$ examples per image class ($1000$ examples in total) for both training and test data. For the neural network model, we consider an $18$-layer ResNet and use its pytorch implementation\footnote{More details are in \url{https://pytorch.org/docs/stable/torchvision/models.html}.} for our experiments. For the random feature model, we use the same $18$-layer ResNet (with randomly initialized weights) to extract random features and only train the last layer. As for the loss function, we use the cross-entropy loss for linear models, random feature models, and neural networks. Furthermore, we analyze locally elastic stability in two settings: locally elastic stability for the whole algorithm and locally elastic stability for a step-wise update of SGD.  

{\bf Locally elastic stability via influence functions.} As shown in Sec. \ref{subsec:IF}, we use influence functions to estimate the quantity $|l(\hat{\theta},z)-l(\hat{\theta}^{\backslash i},z)|$ for all $i$'s in Eq. (\ref{eq:IF}). Similar to \citet{koh2017understanding}, we compared the ResNet-18 with all but the top layer frozen\footnote{We pre-train the model on the whole CIFAR-10 dataset first and keep the pre-trained weights.}, and a random feature model based on a randomly initialized ResNet-18 to a linear model in our experiments. In the experiments for locally elastic stability via influence functions, we add the $\ell_2$ regularization ($\frac{\lambda \|\theta\|^2}{2}$) with $\lambda=1e^{-7}$. We train the last layer (randomly initialized) of the ResNet-18, the random feature model, and the linear model using Adam \citep{kingma2015adam} with learning rate $3e^{-4}$ for $50$ epochs, learning rate $1.0$ for $500$ epochs\footnote{It is worthwhile to note that it is hard for the random feature model, especially based on large neural networks, to converge. For the random feature model, we also use a widely-used learning rate decay, where the initial learning rate is annealed by a factor of $10$ at $1/3$ and $2/3$ during training.}, and learning rate $3e^{-4}$ for $60$ epochs each, and the mini-batch sizes are $50$, $20$, $50$ respectively. The training accuracy for the ResNet-18, the random feature model, and the linear model is $99.3\%$, $94.7\%$, and $94.7\%$, and the test accuracy for them is $93.1\%$, $29.8\%$, and $27.3\%$. The class-level sensitivity approximated by influence function for the neural networks, the random feature model, and the linear model on CIFAR-10 is shown in Fig. \ref{fig:IF-results-appx}. Furthermore, we also consider the influence based on sign values, $l(\hat{\theta},z)-l(\hat{\theta}^{\backslash i},z)$ instead of absolute values $|l(\hat{\theta},z)-l(\hat{\theta}^{\backslash i},z)|$ in Eq. (\ref{eq:IF}), and the corresponding class-level sensitivity is shown in Fig. \ref{fig:IF-sign-results-appx}. 

{\bf Stepwise characterization of locally elastic stability.} To provide the stepwise characterization of locally elastic stability, we consider the trained parameters of SGD with different number of training epochs.  Note that we didn't make any approximation for the experiments in this part. In the training stage, we train the ResNet-18\footnote{Note that we remove the batch normalization for the experiments in step-wise characterization of locally elastic stability for SGD (only in this part).}, the random feature model, and the linear model using SGD with learning rate $0.05$, $1.0$, and $0.3$ separately, and the mini-batch sizes are $50$, $20$, $50$ respectively. The training accuracy (test accuracy) for the ResNet-18 at epoch $0$, $10$, $50$, $100$ are $10.3\%$ ($10.6\%$), $22\%$ ($20.2\%$), $37.6\%$ ($24.3\%$), and $99.9\%$ ($38.9\%$). Similarly, the training accuracy (test accuracy) for the random feature model\footnote{Because it is hard for the random feature model to converge, we use the Adam optimizer and a widely-used learning rate decay for the random feature model, where the initial learning rate is annealed by a factor of $10$ at $1/3$ and $2/3$ during training.} at epoch $0$, $50$, $250$ are $10.3\%$ ($10.6\%$), $63.1\%$ ($28.0\%$), and $90.2\%$ ($30.1\%$).
Similarly, the training accuracy (test accuracy) for the linear model at epoch $0$, $10$, $50$, $100$ are $9.6\%$ ($7.6\%$), $53.4\%$ ($20.5\%$), $98.2\%$ ($22.7\%$), and $100\%$ ($23\%$). To compute the class-level sensitivity $C(c_a, c_b)$, we use the small probing learning rate $1e^{-6}$. The corresponding class-level sensitivity based on absolute values($|l(\hat{\theta},z)-l(\hat{\theta}^{\backslash i},z)|$) and sign values ($l(\hat{\theta},z)-l(\hat{\theta}^{\backslash i},z)$) are shown in Fig. \ref{fig:SGD-results-appx} and Fig.  \ref{fig:SGD-sign-results-appx}.  



{\bf Comparison among $M_\beta$, $\sup_{z'\in \cZ}\bE_{z} \beta(z', z)$ and $M_l$ on a $2$-layer NNs.} We consider a $2$-layer NNs in the following format: 
$$f(W, a, x) = \frac{1}{k}\sum_{r=1}^k a_r \sigma(W_r^Tx)$$
where $d$ is the input dimension, $k$ is the dimension of the hidden layer, and $\sigma$ is the ReLU activation function. As for the loss function, we use the square loss with $\ell_2$ regularization as follows:
$$L(W, a, x) = (f(W, a, x) - y)^2 + \frac{\lambda}{2}(\|W\|_2^2 + \|a\|_2^2)$$
In our experiments, the value of each dimension of $W$, $a$, $x$ is in $[-1, 1]$, and the value of $y$ is in $\{-1, 1\}$. As for the data distribution, each dimension of $x$ is sampled from a uniform distribution on $[-0.5, 1]$ for positive samples with label $y = 1$. Similarly, each dimension of $x$ is sampled from a uniform distribution on $[-1.0, 0.5]$ for negative samples with label $y=-1$. We randomly sample a total of $m = 10000$ examples equally from positive and negative data distribution for both training and test data. As for the hyper parameters, we use $d=10$, $k=100$, and $\lambda = 1e^{-6}$ in our experiments. We trained the $2$-layer NNs $50$ epochs with SGD on batches. The corresponding learning rate and batch size are $1.0$ and $100$.

In this setting, the upper bound of the loss function $M_l$ is $121.00055$ and $M_\beta = m\beta^{\textnormal{U}}_m = \sup_{z'\in S,z\in \cZ}\beta(z',z)$ estimated by the influence function as shown in Eq. (\ref{eq:IF}) is $3464.97$. We can see that  $M_\beta$ is about $29$ times of $M_l$. Similarly, $\sup_{z'\in \cZ}\bE_{z}\beta(z',z)$ estimated by the influence function is $22.91$. It indicates that $\sup_{z'\in \cZ}\bE_{z}\beta(z',z)$ in the locally elastic stability is smaller than $M_l$ and much smaller than $M_\beta = m\beta^{\textnormal{U}}_m = \sup_{z_j\in S,z\in \cZ}\beta(z,z_j)$.

{\bf Class-level locally elastic stability. } In this part,  we consider the case where the sensitivity from one class to another class is the maximum sensitivity instead of the mean sensitivity (in Fig. \ref{fig:IF-results-appx}) among the $100 \times 100$ pairs. In this setting, we have $\sup_{z_j\in S,z\in \cZ}\beta_m(z,z_j)=3.05$ and  $\sup_{z\in\cZ}\bE_{z_j}\beta_m(z,z_j)=0.73$ for the neural networks, and $\sup_{z_j\in S,z\in \cZ}\beta_m(z,z_j)=314$, $\sup_{z\in\cZ}\bE_{z_j}\beta_m(z,z_j)=210$ for the linear model. Furthermore, the maximum and the mean of the 
diagonal (off-diagonal) elements are $3.05$ and $1.02$ ($1.65$ and $0.21$) for the neural networks in this setting. Similarly, the maximum and the mean of the diagonal (off-diagonal) elements are $168$ and $77$ ($314$ and $82$) for the linear model.

\begin{table}[t]
\centering
\scalebox{0.95}{
\begin{tabular}{|c|c|c|c|c|c|c|}
\hline
Positive Examples & {\small Cat, Dog} & {\small Deer, Horse} & {\small Deer, Frog} & {\small Car, Cat} & {\small Plane, Cat } \\ \hline
Negative Examples & {\small Car, Truck} & {\small Car, Truck} & {\small Ship, Truck} & {\small Plane, Bird} & {\small Car, Truck} \\ \hline \hline
 
Between Finer Classes (sign) &0.03 &0.03 &0.02 & 0.06 & 0.08 \\ \hline 
Within Finer Classes (sign) & 0.05 &0.09 & 0.08 & 0.24 & 0.15 \\ \hline \hline
Between Finer Classes (absolute) & 1.52 & 1.90 & 2.05 & 4.50 & 3.19 \\ \hline
Within Finer Classes (absolute) & 1.53 & 1.92 & 2.09 & 4.66 & 3.23 \\ \hline
\end{tabular}}
\caption{A fine-grained analysis of the sensitivity within superclasses (within fine-grained classes or between fine-grained classes) for binary classification.}
\label{table:fine-grained}
\end{table}

{\bf Fine-grained analysis.} To better understand the locally elastic stability, we also provide some fine-grained analysis on the CIFAR-10 dataset. In this part, we consider binary classification on two superclasses, and each superclass is composed of two fine-grained classes. In the training data, we randomly sample $500$ examples for each fine-grained class ($2000$ examples in total). In the test data, we have $1000$ examples for each fine-grained class, so there are $4000$ examples in total. We repeat our experiments five times on different compositions of positive and negative examples as shown in Table \ref{table:fine-grained}. In this part, we still use ResNet-18 as our model, and the sensitivity is approximated by the influence function. We first train the ResNet-18 from scratch and then froze the weights except for the top layer. After that, we train the last layer (randomly initialized) of the ResNet-18, and the class-level sensitivity is shown in Table \ref{table:fine-grained}. \textit{The results show that examples within fine-grained classes (within the same superclass) have stronger sensitivity than examples between fine-grained classes (within the same superclass) for neural networks.}




\end{document}